\documentclass{article}

     \PassOptionsToPackage{numbers, compress}{natbib}

\usepackage[final]{neurips_2020}

\usepackage[utf8]{inputenc} 
\usepackage[T1]{fontenc}    
\usepackage{hyperref}       
\usepackage{url}            
\usepackage{booktabs}       
\usepackage{amsfonts}       
\usepackage{nicefrac}       
\usepackage{microtype}      
\usepackage{floatrow}
\newfloatcommand{capbtabbox}{table}[][\FBwidth]

\usepackage{amsmath}
\usepackage{amssymb}
\usepackage{amsthm}
\usepackage[colorinlistoftodos]{todonotes}
\usepackage{dsfont} 
\usepackage{algorithm,algorithmic}
\usepackage{comment}
\title{On the Similarity between the Laplace\\ and Neural Tangent Kernels}

\author{Amnon Geifman$^1$ \And Abhay Yadav$^2$ \And Yoni Kasten$^1$ \AND 
Meirav Galun$^1$ \hspace*{1.4cm} David Jacobs$^2$ \hspace*{1.1cm} Ronen Basri$^1$\\~\\
$^1$Department of Computer Science, Weizmann Institute of Science, Rehovot, Israel\\
\scriptsize{\texttt{\{amnon.geifman, yoni.kasten, meirav.galun, ronen.basri\}@weizmann.ac.il}}\\[0.1cm]
$^2$Department of Computer Science, University of Maryland, College Park, MD\\
\scriptsize{\texttt{\{jaiabhay,djacobs\}@cs.umd.edu}}}

\newtheorem{theorem}[]{Theorem}
\newtheorem{lemma}[]{Lemma}

\newtheorem{definition}{Definition}[]
\newtheorem{corollary}{Corollary}[]

\newcommand{\Real}{\mathbb{R}}
\newcommand{\Rd}{\Real^d}
\newcommand{\Sphere}{\mathbb{S}}
\newcommand{\Spdm}{\Sphere^{d-1}}

\newcommand{\norm}[1]{\left\lVert#1\right\rVert}
\newcommand{\abs}[1]{\left\vert#1\right\rvert}

\newcommand{\f}{\mathbf{f}}
\newcommand{\g}{\mathbf{g}}

\newcommand{\tr}{\mathbf{t}}

\newcommand{\w}{\mathbf{w}}
\newcommand{\X}{\mathcal{X}}
\newcommand{\x}{\mathbf{x}}
\newcommand{\y}{\mathbf{y}}
\newcommand{\z}{\mathbf{z}}

\newcommand{\hh}{{\mathcal{H}}}

\newcommand{\HNTKFC}{{\hh^{\mathrm{FC_\beta(2)}}}}
\newcommand{\HNTKFCb}{{\hh^{\mathrm{FC_\beta(L)}}}}
\newcommand{\HNTKFCbl}[1]{\hh^{\mathrm{FC_\beta(#1)}}}
\newcommand{\HExp}{{\hh^{\mathrm{Lap}}}}
\newcommand{\HGauss}{{\hh^{\mathrm{Gauss}}}}

\newcommand{\lExp}{\lambda^{\mathrm{Lap}}}
\newcommand{\lNTKFCbTwo}{\lambda^{\mathrm{FC_\beta(2)}}}
\newcommand{\lNTKFCb}{\lambda^{\mathrm{FC_\beta(L)}}}

\newcommand{\kr}{\boldsymbol{k}}
\newcommand{\krfc}{\kr^{\mathrm{FC_0(L)}}}
\newcommand{\krfcl}[1]{\kr^{\mathrm{FC_0(#1)}}}
\newcommand{\krfcb}{\kr^{\mathrm{FC_\beta(L)}}}
\newcommand{\krfcbl}[1]{\kr^{\mathrm{FC_\beta(#1)}}}
\newcommand{\krb}{\kr^{\mathrm{Bias(L)}}}
\newcommand{\krbtwo}{\kr^{\mathrm{Bias(2)}}}
\newcommand{\krbl}[1]{\kr^{\mathrm{Bias(#1)}}}
\newcommand{\krfctwo}{\kr^{\mathrm{FC_\beta(2)}}}

\newcommand{\krres}{\kr^{\mathrm{Res_0(L)}}}
\newcommand{\krexp}{\kr^{\mathrm{Lap}}}
\newcommand{\krhexp}{\kr^{\mathrm{HLap}}}
\newcommand{\krgamma}{\kr^{\mathrm{\gamma}}}

\newcommand{\rb}[1]{\textcolor{purple}{[Ronen: #1]}}
\newcommand{\ag}[1]{\textcolor{green}{[Amnon: #1]}}

\newcommand{\abh}[1]{\textcolor{violet}{[Abhay: #1]}}

\newcommand{\cmt}[1]{}
\newcommand{\resnet}[1]{}  
\newcommand{\broaderimpact}[1]{}

\begin{document}

\maketitle

\begin{abstract}
    Recent theoretical work has shown that massively overparameterized neural networks are equivalent to kernel regressors that use {\em Neural Tangent Kernels} (NTKs).  Experiments show that these kernel methods perform similarly to real neural networks.  Here we show that NTK for fully connected networks with ReLU activation is closely related to the standard Laplace kernel.  We show theoretically that for normalized data on the hypersphere both kernels have the same eigenfunctions and their eigenvalues decay polynomially at the same rate, implying that their Reproducing Kernel Hilbert Spaces (RKHS) include the same sets of functions.  This means that both kernels give rise to classes of functions with the same smoothness properties. The two kernels differ for data off the hypersphere,  but experiments indicate that when data is properly normalized these differences are not significant. Finally, we provide experiments on real data comparing NTK and the Laplace kernel, along with a larger class of $\gamma$-exponential kernels.  We show that these  perform almost identically.  Our  results suggest that much insight about neural networks can be obtained from analysis of the well-known Laplace kernel, which has a simple closed form.
\end{abstract}

\section{Introduction}

Neural networks with significantly more parameters than training examples have been successfully applied to a variety of tasks.  Somewhat contrary to common wisdom, these models typically generalize well to unseen data.  It has been shown that in the limit of infinite model size, these neural networks are equivalent to kernel regression using a family of novel {\em Neural Tangent Kernels} (NTK) \cite{jacot2018neural,arora2019exact}. NTK methods can be analyzed to explain many properties of neural networks in this limit, including their convergence in training and ability to generalize \cite{belkin2018overfitting,belkin2018understand,bordelon2020spectrum,liang2018just}. Recent experimental work has shown that in practice, kernel methods using NTK perform similarly, and in some cases better, than neural networks \cite{Arora2020Harnessing}, and that NTK can be used to accurately predict the dynamics of neural networks \cite{allen-zhu2019,arora2019exact,ronen2019convergence}. This suggests that a better understanding of NTK can lead to new ways to analyze neural networks.

These results raise an important question: Is NTK significantly different from standard kernels?  For the case of fully connected (FC) networks, \cite{Arora2020Harnessing} provides experimental evidence that NTK is especially effective, showing that it outperforms the Gaussian kernel on a large suite of machine learning problems.  Consequently, they argue that NTK should be added to the standard machine learning toolbox.  \cite{belkin2018understand} has shown empirically that the dynamics of neural networks on randomly labeled data more closely resembles the dynamics of learning through stochastic gradient descent with the Laplace kernel than with the Gaussian kernel.  In this paper we show theoretically and experimentally that NTK does closely resemble the Laplace kernel, already a standard tool of machine learning.

\begin{figure}[t]
    \centering
    \includegraphics[height=2.5cm]{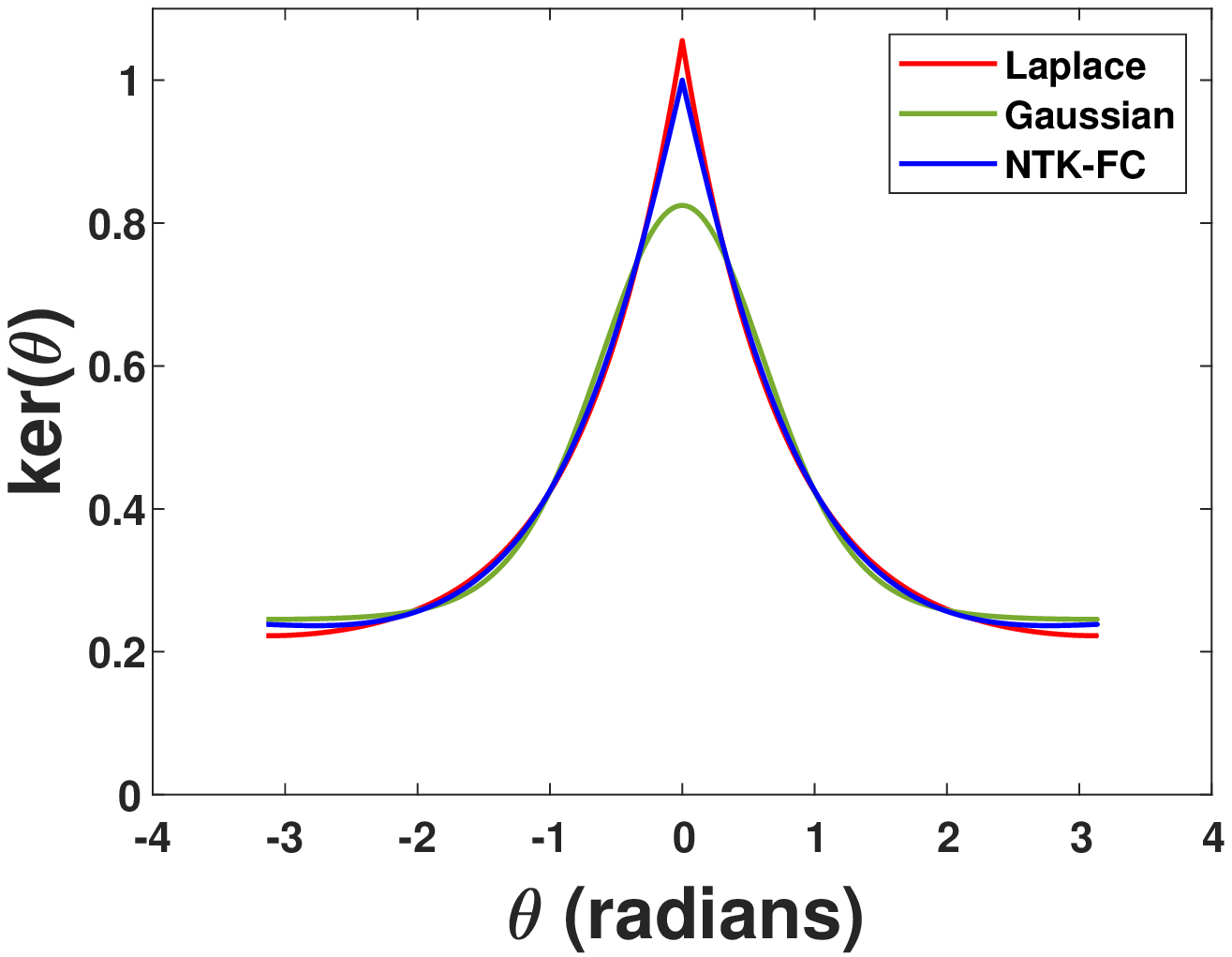}
    \includegraphics[height=2.5cm]{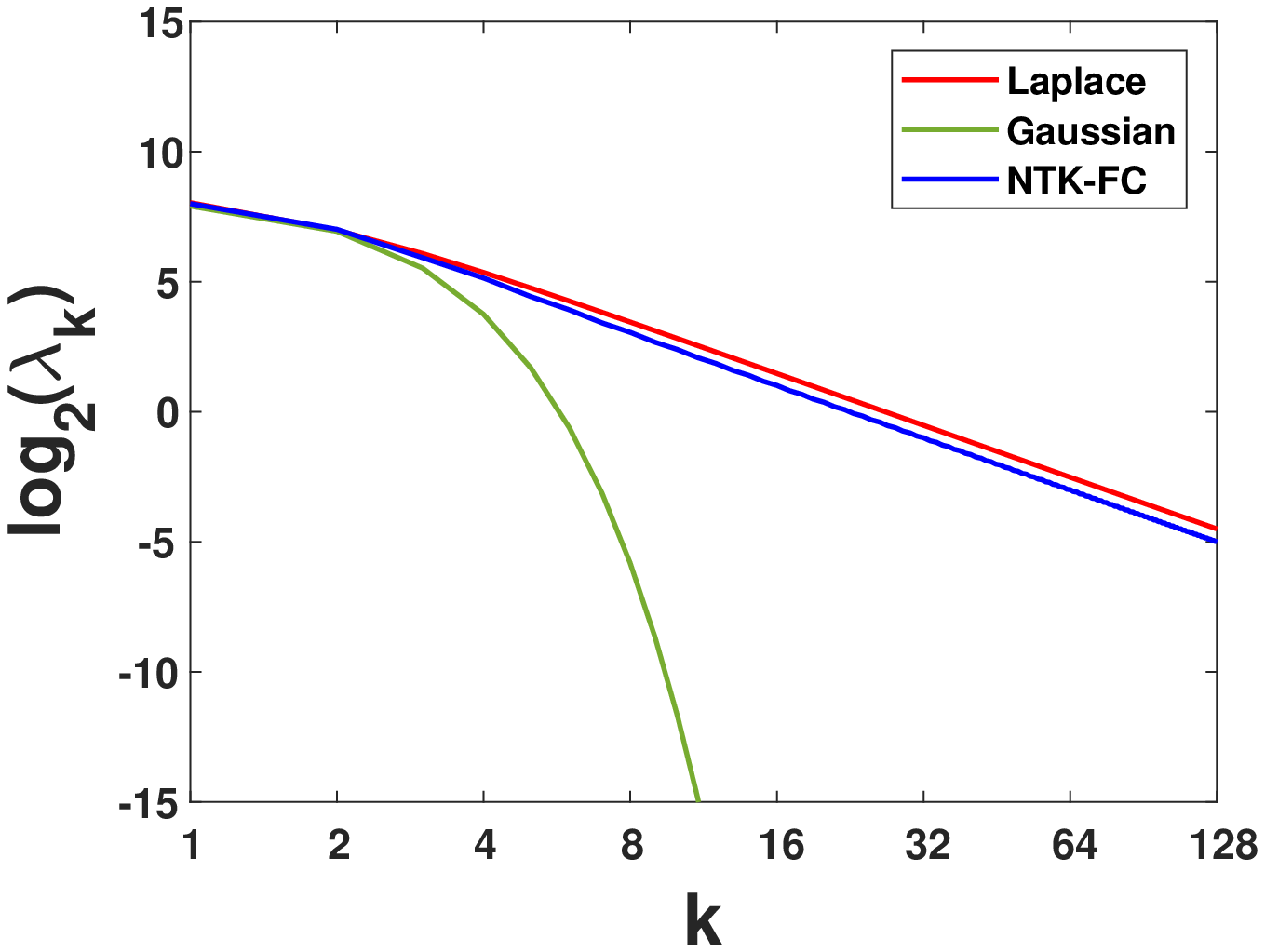}
    \includegraphics[height=2.5cm]{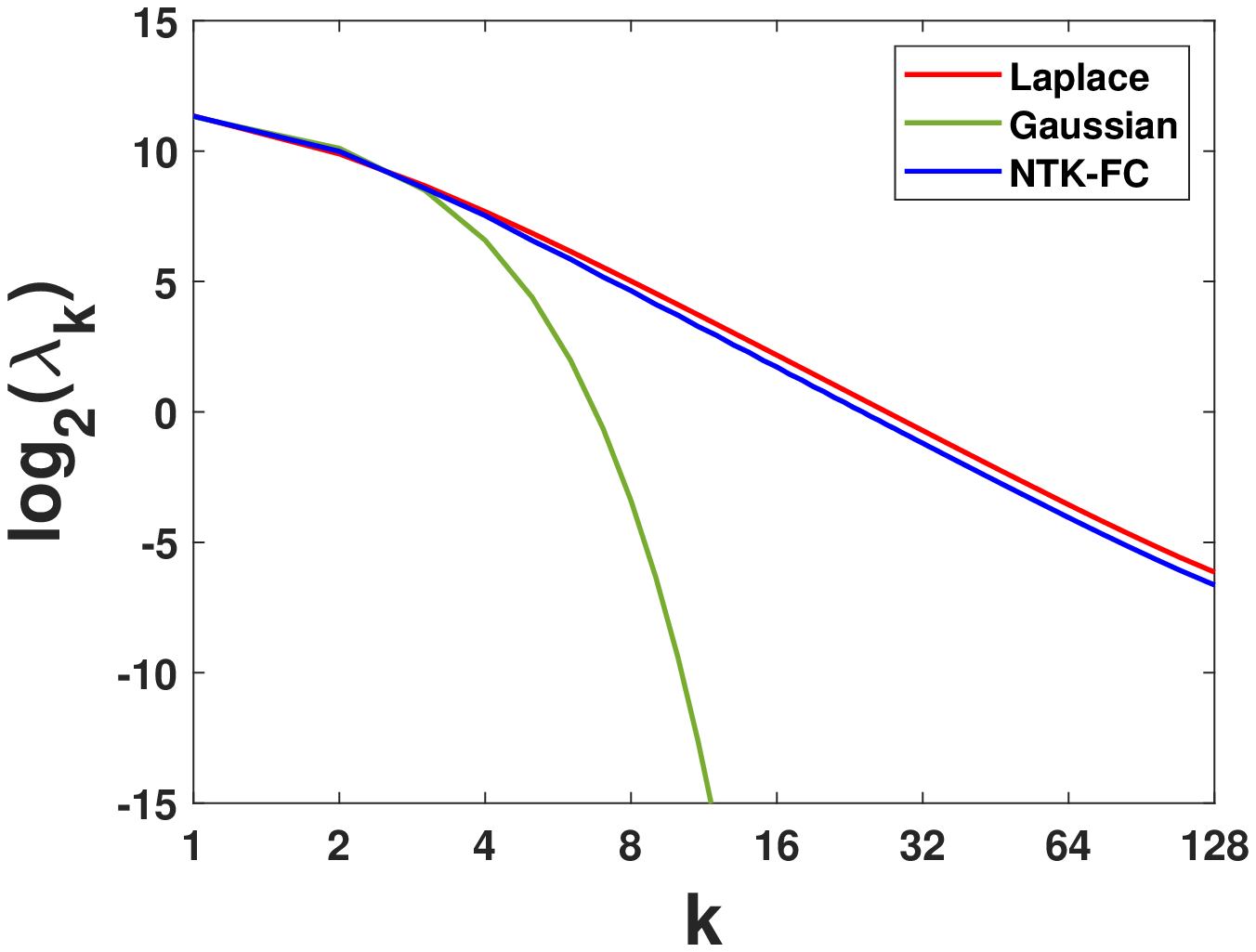}
    \includegraphics[height=2.5cm]{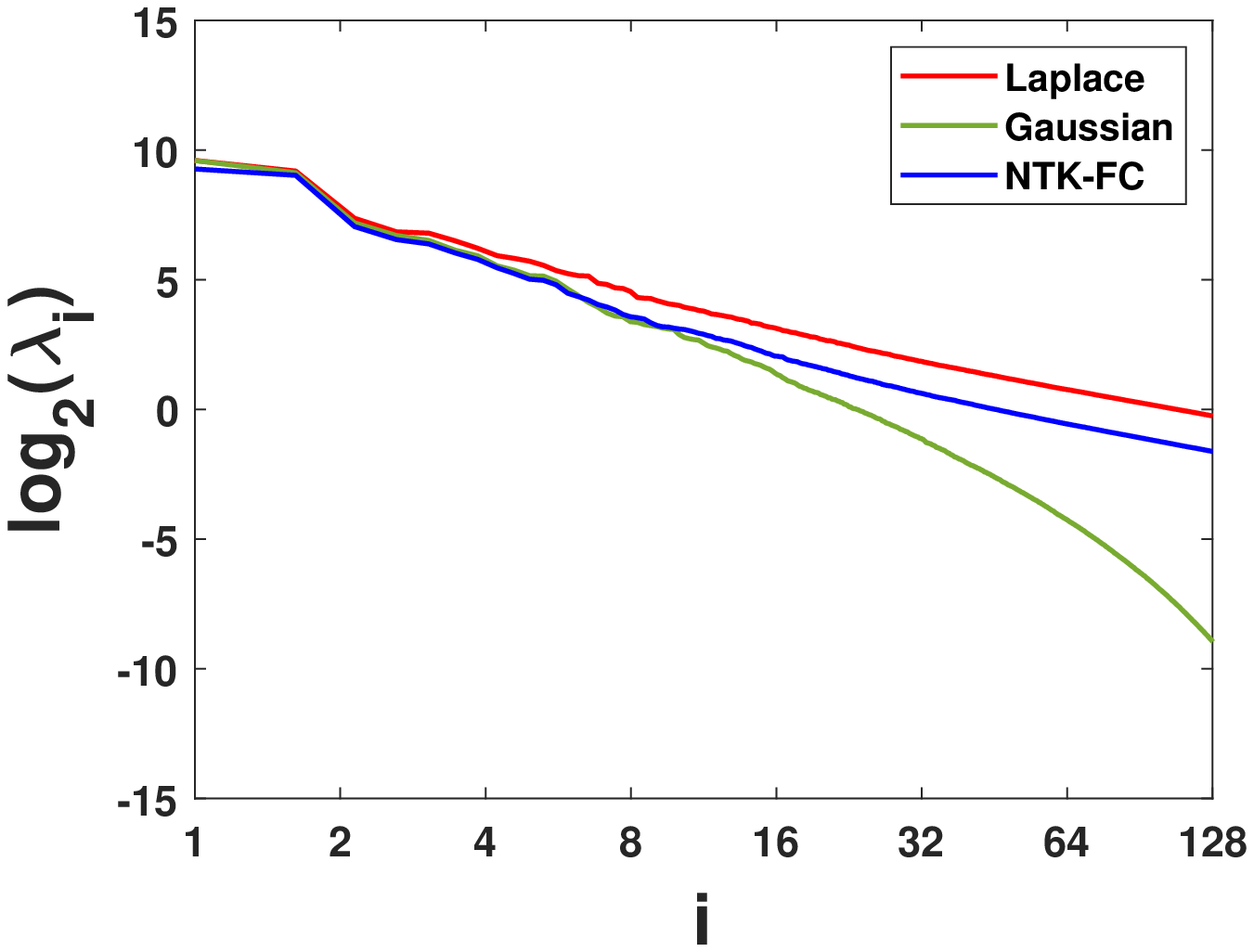}
    
    \caption{\small Left: An overlay of the NTK for a 6-layer FC network with ReLU activation with the Laplace and Gaussian kernels,  as a function of the angle between their arguments. The exponential kernels are modulated by an affine transformation to achieve a least squares fit to the NTK. Note the high degree of similarity between the Laplace kernel and NTK. Middle left: eigenvalues as a function of frequency in $\Sphere^1$. The slopes in these log-log plots indicate the rate of decay, which is similar for both the Laplace kernel and for NTK for the FC network \resnet{and residual networks }with 6 layers. (Empirical slopes are -1.94 for both Laplace and NTK-FC\resnet{ and NTK-Res}.) The eigenvalues of the Gaussian kernel, in contrast, decay exponentially. Middle right: Same for $\Sphere^2$. (Empirical slopes are -2.75 for the Laplace and NTK-FC\resnet{ and -2.78 for NTK-Res}.) Right: Same estimated for the UCI Abalone dataset (here we show eigenvalues as function of eigenvalue index).}
    \label{fig:exponent_uniform_resnet}
\end{figure}

Kernels are mainly characterized by their corresponding Reproducing Kernel Hilbert Space (RKHS), which determines the set of functions they can produce \cite{kanagawa2018gaussian}. They are further characterized by the RKHS norm they induce, which is minimized (implicitly) in every regression problem. Our main result is that when restricted to the hypersphere $\Spdm$, NTK for a fully connected (FC) network with ReLU activation and bias has the same RKHS as the Laplace kernel, defined as $\krexp(\x,\z)=e^{-c\|\x-\z\|}$ for points $\x,\z \in \Spdm$ and constant $c>0$. (In general, NTK for deeper networks is more sharply peaked, corresponding to larger values of $c$, see supplementary material.) This equivalence of RKHSs is shown by establishing that on the hypersphere the eigenfunctions of NTK and the Laplace kernels coincide and their eigenvalues decay at the same rate (see Figure \ref{fig:exponent_uniform_resnet}), implying in turn that gradient descent (GD) with both kernels should have the same dynamics, explaining \cite{belkin2018understand}'s experiments. In previous work, the eigenfunctions and eigenvalues of NTK have been derived on the hypersphere for networks with only one hidden layer, while these properties of the Laplace kernel have been studied in $\Rd$.  We derive new results for the Laplace kernel on the hypersphere, and for NTK for deep networks on the hypersphere and  in $\Rd$. In $\Rd$, NTK gives rise to radial eigenfunctions, forgoing the shift invariance property of exponential kernels.  Experiments indicate that this difference is not significant in practice.

Finally, we show experiments indicating that the Laplace kernel achieves similar results to those obtained with NTK on real-world problems. We further show that by using the more general, $\gamma$-exponential kernel~\cite{GammaExponential}, which allows for one additional parameter, $\krgamma(\x,\z)=e^{-c\|\x-\z\|^\gamma}$, we achieve slightly better performance than NTK on a number of standard datasets.

\section{Related Works}

The connection between neural networks and kernel methods has been investigated for over two decades. Early works have noted the equivalence between neural networks with single hidden layers of infinite width and Gaussian Processes (GP) \cite{williams1997computing,neal2012bayesian}, where GP prior can be used to achieve exact Bayesian inference. Recently \cite{lee2017deep,matthews2018gaussian} have extended the results to deep fully-connected neural networks in which all but the last layer retain their initial values. In this context, \cite{cho2009kernel} introduced the Arc-cosine kernel, while \cite{daniely2016toward} showed a duality between neural networks and compositional kernels. 

More recent work introduced the family of neural tangent kernels (NTK) \cite{jacot2018neural,arora2019exact}. This work showed that for massively overparameterized and fully trained networks, their training dynamics closely follows the path of kernel gradient descent, and that training converges to the solution of kernel regression with NTK. Follow-up work defined analogous kernels for residual \cite{huang2020deep} and convolutional networks \cite{arora2019exact,li2019enhanced}. Recent work also showed empirically that classification with NTK achieves performance similar to deep neural networks with the corresponding architecture \cite{arora2019exact,li2019enhanced}.

The equivalence between kernels and overparameterized neural networks opened the door to studying inductive bias in neural networks. For two layer, FC networks,
\cite{bach2017breaking,ronen2019convergence,cao2019towards} investigated the spectral property of the NTK when the data is distributed uniformly on the hypersphere, showing in particular that with GD low frequencies are learned before higher ones.  \cite{basriNonuniform} extended these results to non-uniform distributions. \cite{yang2019fine} analyzed the eigenvalues of NTK over the Boolean cube, and \cite{fan2020spectra} analyzed its spectrum under approximate pairwise orthogonality. \cite{bach2017breaking,bietti2019inductive} further leveraged the spectral properties of the kernels to investigate their RKHS in the case of bias-free two layer networks. Our results apply to deep networks with bias. \cite{ghorbani2019linearized} studied approximation bounds for two layer neural networks, and \cite{belkin2018overfitting,belkin2018understand,bordelon2020spectrum,liang2018just} studied generalization properties of kernel methods in the context of neural networks.

Positive definite kernels and their associated RKHSs have been studied extensively, see, e.g., \cite{saitoh2016theory,stein2016introduction} for reviews. The spectral properties of classic kernels, e.g., the Gaussian and Laplace kernels, are typically derived for input in $\mathbb{R}^d$ \cite{kimeldorf1970correspondence}. Several papers examine the RKHS of common kernels (e.g., the Gaussian and polynomial) on the hypersphere \cite{minh2006mercer,narcowich2007approximation,narcowich2002scattered}.  

Recent work compares the performance of NTK to that of common kernels. Specifically, \cite{Arora2020Harnessing}'s experiments suggest that NTK is superior to the Gaussian and low degree polynomial kernels. \cite{belkin2018understand} compares the learning speed of GD for randomly mislabeled data, showing that NTK learns such data as fast as the Laplace kernel and much faster than the Gaussian kernel. Our analysis provides a theoretical justification of this result.

\section{NTK vs.\ the Exponential Kernels}

Our aim is to compare NTK to common kernels. In comparing kernels we need to consider two main properties: first, what functions are included in their respective RKHS and secondly, how their respective norms behave. (These concepts are reviewed below in Sec.~\ref{sec:preliminaries}.) The answer to the former question determines the set of functions considered for regression, while the answer to the latter determines the result of regression. Together, these will determine how a kernel generalizes to unseen data points. Below we see that on the hypersphere both NTK and exponential kernels (e.g., Gaussian and Laplace) give rise to the same set of eigenfunctions. Therefore, the answers to the questions above are determined fully by the corresponding eigenvalues. Moreover, the asymptotic decay rate of the eigenvalues of each kernel determines their RKHS.

As an example consider the exponential kernels in $\Rd$, i.e., the kernels $e^{-c\norm{\x-\z}^{\gamma}}$, where $c>0$ and $0 < \gamma \leq 2$~\cite{GammaExponential}. These shift invariant kernels have the Fourier transform as their eigenfunctions. The eigenvalues of the Gaussian kernel, i.e.,  $\gamma =2$, decay exponentially, implying that its respective RKHS includes only infinitely smooth functions. In contrast, the eigenvalues of the Laplace kernel, i.e., $\gamma=1$, decay polynomially, forming a space of continuous, but not necessarily smooth functions. 

Our main theoretical result is that when restricted to the hypersphere $\Spdm$
\begin{equation*}
    \HGauss \subset \HExp = \HNTKFC \subseteq \HNTKFCb,
\end{equation*}
where $\HGauss$ and $\HExp$ denote the RKHSs associated with the Gaussian and Laplace kernels, and $\HNTKFCb$ denotes the NTK for a FC network with $L$ layers, ReLU activation, and bias. Further empirical results indicate that $\HExp = \HNTKFCb$ for the entire range $L \ge 2$\resnet{ for fully connected networks, as well as for residual networks}.
Indeed, the subsequent work of~\cite{chen2020deep} proves that $\HNTKFC \supseteq \HNTKFCb$, thus together with our results proving that $\HNTKFC = \HNTKFCb$.

Next we briefly recall basic concepts in kernel regression. We subsequently characterize the RKHS of NTK and the Laplace kernel and show their equivalence in $\Spdm$. Finally, we discuss how these kernels extend outside of the sphere to the entire $\Rd$ space. All lemmas and theorems are proved in the supplementary material.

\subsection{Preliminaries}  \label{sec:preliminaries}

We consider positive definite kernels $\kr: \X \times \X \rightarrow \Real$ defined over a compact metric space $\X$ endowed with a finite Borel measure $\mathcal{V}$. Each such kernel is associated with a Reproducing Kernel Hilbert Space (RKHS) of functions, $\hh$, which includes the set of functions the kernel reproduces, i.e., $f(\x)=\langle f,\kr(\cdot,\x)\rangle_\hh$ where the inner product is inherited from the respective Hilbert space. 
For such kernels the following holds:
\begin{enumerate}
    \item For all $x\in \mathcal{X}$ we have that the  $\kr(\cdot,x)\in \mathcal{H}$.
    \item Reproducing property: for all $x\in \mathcal{X}$ and for all $f\in \mathcal{H}$ it holds that $f(x)=\langle f,\kr(\cdot,x)\rangle_{\mathcal{H}}$.
\end{enumerate}
Moreover, RKHSs and positive definite kernels are uniquely paired.

According to Mercer's theorem $\kr$ can be written as
\begin{align}  \label{eq:mercer}
    \kr(\x,\z)=\sum_{i\in I} \lambda_i \Phi_i(\x) \Phi_i(\z), ~~~ \x,\z \in {\cal X},
\end{align}
where $\{(\lambda_i,\Phi_i)\}_{i\in I}$ are the eigenvalues and eigenfunctions of $\kr$ with respect to the measure $\mathcal{V}$, i.e.,
\begin{align*}
    \int k(\x,\z) \Phi_i(\z) d\mathcal{V}(\z) = \lambda_i \Phi_i(\x).
\end{align*}
The RKHS $\hh$ is the space of functions $f \in \hh$ of the form $f(\x)=\sum_{i \in I} \alpha_i \Phi_i(\x)$ whose RKHS norm is finite, i.e., $\|f\|_\hh=\sum_{i \in I} \frac{\alpha_i^2}{\lambda_i}<\infty$. The latter condition restricts the set of functions in an RKHS, allowing only functions that are sufficiently smooth in accordance with the asymptotic decay of $\lambda_k$.

The literature considers many different kernels (see, e.g., \cite{Genton2001}). Here we discuss the family of $\gamma$-exponential kernels $\krgamma(\x,\z)=e^{-c\|\x-\z\|^\gamma}$, $0< \gamma \leq 2$, which include the Laplace ($\gamma=1$) and the Gaussian ($\gamma=2$) kernels.

\textbf{Neural Tangent Kernel}. Let $f(\theta,\x)$ denote a neural network function with ReLU activation and trainable parameters $\theta$. Then the corresponding NTK is defined as
\begin{equation*}
    \kr^{\mathrm{NTK}}(\x,\z) = \mathbb{E}_{\theta \sim \mathcal{P}} \left< \frac{\partial f(\theta,\x)}{\partial \theta}, \frac{\partial f(\theta,\z)}{\partial \theta} \right>,
\end{equation*}
where expectation is taken over the probability distribution $\mathcal{P}$ of the initialization of $\theta$, and we assume that the width of each layer tends to infinity. Our results focus \resnet{mainly }on NTK kernels corresponding to deep, fully connected network architectures that may or may not include bias, where bias, if it exists, is initialized at zero. We denote these kernels by $\krfc$ for the bias-free version and $\krfcb$ for NTK with bias and define them in the supplementary material. 
\resnet{ We further provide some theoretical and empirical results for NTK corresponding to residual networks, denoted $\krres$, defined in the supplementary material.} 

\textbf{Kernel regression}. Given training data $\{(\x_i,y_i)\}_{i=1}^n$, $\x_i \in \X$, $y_i \in \Real$, kernel ridge regression is the solution to
\begin{equation}  \label{eq:regression}
    \min_{f \in \hh}\sum_{i=1}^n(f(\x_i)-y_i)^2+\lambda \|f\|^2_\hh. 
\end{equation}
When $\lambda \rightarrow 0$ this problem is called minimum norm interpolant, and the solution satisfies
\begin{equation}
    \min_{f \in \hh} \|f\|_\hh ~~~ \mathrm{s.t.} ~~~~ \forall i, ~ f(\x_i)=y_i.
\end{equation}

The solution of \eqref{eq:regression} is given by $f(\x)=\kr_\x^T (K+\lambda I)^{-1} \y$, where the entries of $\kr_\x \in \Real^n$ are $\kr(\x,\x_i)$, $K$ is the $n \times n$ matrix with $K_{ij}=\kr(\x_i,\x_j)$, $I$ denotes the identity matrix, and $\y=(y_1,...,y_n)^T$.
Further review of kernel methods can be found, e.g., in \cite{kanagawa2018gaussian,saitoh2016theory}.

\subsection{NTK in $\Spdm$}
\label{sec:NTKonsphere}

We next consider the NTK for fully connected networks applied to data restricted to the hypersphere $\Spdm$. To characterize the kernel, we first aim to determine the eigenvectors of NTK. This will be a direct consequence of Lemma~\ref{lemma:NTK_sphere}. Subsequently in Theorem~\ref{thm:NTKdecay} we will characterize the decay rate of the corresponding eigenvalues.

\begin{lemma}
\label{lemma:NTK_sphere}
    Let $\krfcb(\x,\z)$, $\x,\z \in \Spdm$, denote the NTK kernels for FC  networks with $L \ge 2$ layers, possibly with bias initialized with zero. This kernel is zonal, i.e., $\krfcb(\x,\z)=\krfcb(\x^T\z)$. (Note the abuse of notation, which should be clear by context.)
\end{lemma}
We note that for the bias-free $\krfc$ this lemma was proven in \cite{basriNonuniform} and we extend the proof 
to allow for bias. It is well known that the spherical harmonics are eigenvectors for any zonal kernel with respect to the uniform measure on $\Sphere^{d-1}$ with $d \ge 3$. (For background on Spherical Harmonics see, e.g., \cite{gallier2009notes}). Therefore, due to Mercer's Theorem \eqref{eq:mercer}, any zonal kernel $\kr$ can be written as
\begin{equation}  \label{eq:sh}
    \kr(\x,\z) = \sum_{k=0}^\infty \lambda_k \sum_{j=1}^{N(d,k)} Y_{k,j}(\x) Y_{k,j}(\z),
\end{equation}
where $Y_{k,j}(.)$ denotes the spherical harmonics of $\Spdm$, $N(d,k)$ denotes the number of harmonics of order $k$ in $\Spdm$, and $\lambda_k$ are the respective eigenvalues. On the circle $\Sphere^1$ the eigenvectors are the Fourier series, and $\kr(\x,\z)=\sum_{k=0}^\infty \frac{1}{c_k} \lambda_k \cos(k\theta)$, where $\theta=\arccos(\x^T\z)$ and $c_k$ is a normalization factor, $c_0=4\pi^2$ and $c_k=\pi^2$ when $k \ge 1$.


Deriving the eigenvalues for NTK for deep networks is complicated, due to its recursive definition. For a two-layer network without bias, \cite{bach2017breaking,bietti2019inductive} proved that the eigenvalues decay at a rate of $O(k^{-d})$. With no bias, however, two-layer networks are nonuniversal, and in particular $\lambda_k=0$ for odd $k \ge 3$~\cite{ronen2019convergence}. 
To avoid this issue Theorem \ref{thm:NTKdecay} establishes that with bias NTK 
is universal for any number of layers $L \ge 2$, and its eigenvalues decay at a rate no faster than $O(k^{-d})$. Moreover, with $L=2$ the eigenvalues decay exactly at the rate of $O(k^{-d})$.

\begin{theorem}  \label{thm:NTKdecay}
Let $\x,\z \in \mathbb{S}^{d-1}$. With bias initialized at zero:
\begin{enumerate}
\item $\krfcb$ decomposes according to \eqref{eq:sh} with $\lambda_k>0$  for all $k \ge 0$, and
\item $\exists k_0$ and constants $C_1,C_2,C_3>0$ that depend on the dimension $d$ such that $\forall k>k_0$
\begin{enumerate}
    \item $C_1 k^{-d} \leq \lambda_k \leq C_2 k^{-d}$ if $L=2$, and 
    \item $C_3 k^{-d} \leq \lambda_k$ if $L \ge 3$.
\end{enumerate}  
\end{enumerate}
\end{theorem}

The proof of this theorem for $L=2$ borrows techniques from \cite{bach2017breaking}. The proof for $L \ge 3$
relies mainly on showing that the algebraic operations in the recursive definition of NTK (including addition, product and composition) do not increase the rate of decay. The consequence of Theorem \ref{thm:NTKdecay} is that NTK for FC networks gives rise to an infinite size feature space and that its eigenvalues decay no faster than $O(k^{-d})$. While our proofs only establish a bound for the case that $L \ge 3$, empirical results suggest that the eigenvalues for these kernels decay exactly as $\Theta(k^{-d})$, as can be seen in Figure~\ref{fig:exponent_uniform_resnet}.

\resnet{
Finally, similar to NTK for FC networks, NTK for bias-free residual networks, denoted $\krres$, is zonal,
and therefore its eigenfunctions in $\Spdm$ are the spherical harmonics. Experiments indicate that its eigenvalues too decay at the rate of $\Theta(k^{-d})$, see Figure~\ref{fig:exponent_uniform_resnet}.
}

\subsection{NTK vs. exponential kernels in $\Spdm$}  \label{sec:ntk_vs_exp}

The polynomial decay of the eigenvalues of NTK suggests that NTK is closely related to the Laplace kernel, as we show next. Indeed, any shift invariant and  isotropic kernel, i.e., $\kr(\x,\y)=\kr(\|\x-\y\|)$, in $\Rd$ is zonal when restricted to the hypersphere, since $\x,\y \in \Spdm$ implies $\|\x-\y\|^2=2(1-\x^T\y)$.  Therefore, in $\Spdm$ the spherical harmonics are the eigenvectors of the exponential kernels.


\cite{minh2006mercer} shows that the Gaussian kernel restricted to the hypersphere yields eigenvalues that decay exponentially fast. In contrast we next prove that the eigenvalues of the Laplace kernel restricted to the hypersphere decay polynomially as $\Theta(k^{-d})$, the same decay rate shown for NTK in Theorem~\ref{thm:NTKdecay} and in Figure~\ref{fig:exponent_uniform_resnet}.

\begin{theorem}  \label{thm:exp_decay}
Let $\x,\z \in \Spdm$ and write the Laplace kernel as $\krexp(\x^T\z)=e^{-c\sqrt{1-\x^T\z}}$, restricted to $\Spdm$. Then $\krexp$ can be decomposed as in \eqref{eq:sh} with the eigenvalues $\lambda_k$ satisfying $\lambda_k>0$, and $\exists k_0$ such that $\forall k>k_0$ it holds that:
\begin{equation*}
    B_1 k^{-d} \leq \lambda_k \leq B_2 k^{-d}
\end{equation*}
where $B_1,B_2>0$ are constants that depend on the dimension $d$ and the parameter $c$.
\end{theorem}

Our proof
uses the decay rate of the Laplace kernel in $\Rd$, and results due to \cite{narcowich2002scattered,narcowich2007approximation} that relate Fourier expansions in $\Rd$ to their corresponding spherical harmonic expansions in $\Spdm$. This allows us to state our main theoretical result.

\begin{theorem}  \label{cor:equivalence}
Let $\HExp$ denote the RKHS for the Laplace kernel restricted to $\Spdm$, and let $\HNTKFCb$ denote the NTK corresponding respectively to a FC network with $L$ layers, ReLU activation, and bias, restricted to $\Spdm$, then $\HExp = \HNTKFC \subseteq \HNTKFCb$.
\end{theorem}

The common decay rates of NTK and the Laplace kernel in $\Spdm$ imply that the set of functions in their RKHSs are identical, having the same smoothness properties. In particular, due to the norm equivalence of RKHSs and Sobolev spaces both spaces include functions that have weak derivatives up to order $d/2$ \cite{narcowich2002scattered}.
We recall that empirical results suggest further that $\krfcb$ decays exactly as $\Theta(k^{-d})$, and so we conjecture that $\HExp = \HNTKFCb$. We note that despite this asymptotic similarity, the eigenvalues of NTK and the Laplace kernel are not identical even if we correct for shift and scale. Consequently, each kernel may behave slightly differently. Our experiments in Section~\ref{sec:experiments} suggest that this results in only small differences in performance.

The similarity between NTK and the Laplace kernel has several implications. First, the dynamics of gradient descent for solving regression \eqref{eq:regression} with both kernels \cite{chapelle2007training} should be similar. For a kernel with eigenvalues $\{\lambda_i\}_{i=1}^\infty$ a standard calculation shows that GD requires $O(1/\lambda_i)$ time steps to learn the $i$th eigenfunction (e.g., \cite{arora2019fine,ronen2019convergence}). For both NTK and the Laplace kernel in $\Spdm$ this implies that $O(k^d)$ time steps are needed to learn a harmonic of frequency $k$. This is in contrast for instance with the Gaussian kernel, where the time needed to learn a harmonic of frequency $k$ grows exponentially with $k$. This in particular explains the empirical results of \cite{belkin2018understand}, where it was shown that fitting noisy class labels with the Laplace kernel or neural networks requires a similar number of SGD steps. The authors of \cite{belkin2018understand} conjectured that ``optimization performance is controlled by
the type of non-smoothness," as indeed is determined by the identical RKHS for NTK and the Laplace kernel. 

The similarity between NTK and the Laplace kernel also implies that they have similar generalization properties. Indeed various generalization bounds rely explicitly on spectral properties of kernels. For example, given a set of training points $X \subseteq \Spdm$ and a target function $f:\Spdm \rightarrow \mathbb{R}$, then the error achieved by the kernel regression estimator given $X$, denoted $\hat{f}_X$, is (see, e.g., \cite{jetter1999error})
\begin{align*}
    \norm{f-\hat{f}_X}_\infty \leq C\cdot h(X)^\alpha \norm{f}_{\mathcal{H}_k}, ~~f\in \mathcal{H}_k,
\end{align*}
where $h(X):=\sup _{\z \in \Spdm} \inf_{\x\in X} \arccos(\z^T \x)$ is the mesh norm of $X$ (and thus depends on the density of the points), and $\alpha$ depends on the smoothness property of the kernel. Specifically, for both the Laplace kernel and NTK $\alpha=1/2$.

Likewise, with $n$ training points and $f\in \hh_{\kr}$, \cite{micchelli1979design} derived the following lower bound
\begin{align*}
    \mathbb{E}_{X}\left((f-\hat{f}_X)^2\right)\geq \sum _{i=n+1}^\infty \alpha_i, 
\end{align*}
where $\alpha_i$ are the eigenvalues of $f$. Both of these bounds are equivalent asymptotically up to a constant for NTK (with bias) and the Laplace kernel.

\cmt{

Figure~\ref{fig:fitting} and Table~\ref{tab:fitting_accuracy} show the result of fitting several functions with the NTK, Laplace and the Gaussian kernels. It can be seen that NTK and Laplace achieve similar fit, and the quality of fit for the three kernels depends on their smoothness properties.

\begin{table*}[h]\tiny{
\begin{tabular}{|l|c|c|c|}
\hline
Kernel & $f(\theta)=sign(\theta)$ & $f(\theta)=e^{-\abs{\theta}}$ & $f(\theta)=e^{-\theta^2}$ \\ \hline

NTK & 0.02 & 0.00039 &  0.00034 \\ \hline

Laplace & 0.02 & 0.00024 & 0.00053 \\ \hline

Gaussian & 0.14 & 0.009 & 0.000029 \\ \hline
\end{tabular}
}
\caption{\small Fitting target functions with the NTK, Laplace and Gaussian kernels using 50 equally spaced points (mean absolute deviation). Note that $sign(\theta)$ is neither of the RKHS for all three kernels, $e^{-\abs{\theta}}$ is in the RKHS of NTK and the Laplace kernel only, and $e^{-\theta^2}$ is in all three RKHS.}
\label{tab:fitting_accuracy}
\end{table*}

\begin{figure}[h]
    \centering
    \includegraphics[width=10cm]{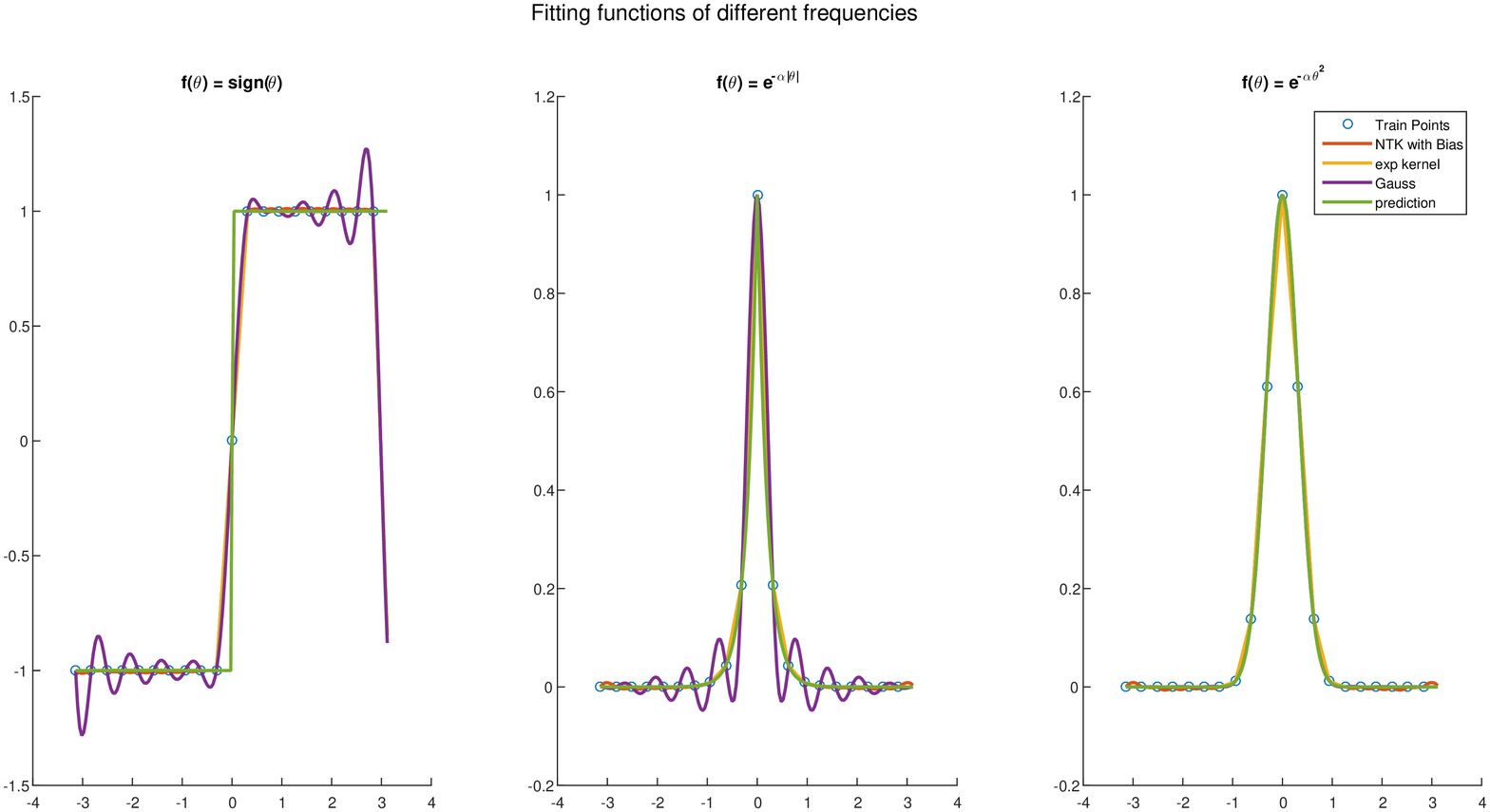}
    \caption{\small Learning function of different frequencies. The training points marked with circles. Each line corresponds to kernel regression on the training data: orange to NTK with bias, yellow to Laplace kernel and purple to Gauss kernel. The green line corresponds to the target function $f(\theta)$\ag{Is it a better figure?}\rb{Not good enough. Let's talk}}
    \label{fig:fitting}
\end{figure}

}

\subsection{NTK vs.\ exponential kernels in $\Rd$}

While theoretical discussions of NTK largely assume the input data is normalized to lie on the sphere, such normalization is not the common practice in neural network applications. Instead, most often each feature is normalized separately by setting its mean to zero and variance to 1. Other normalizations are also common. It is therefore important to examine how NTK behaves outside of the hypersphere, compared to common kernels. 

Below we derive the eigenfunctions of NTK for deep FC networks with ReLU activation with and without bias. We note that \cite{bach2017breaking,bietti2019inductive} derived the eigenfunctions of NTK for two-layer FC networks with no bias. We will show that the same eigenfunctions are obtained with deep, bias-free networks, and that additional eigenfunctions appear when bias is added. We begin with a definition.

\begin{definition}
A kernel $\kr$ is homogeneous of order $\alpha$ if $\kr(\x,\z)=\|\x\|^\alpha \|\z\|^\alpha \kr\left(\frac{\x^T\z}{\|\x\|\|\z\|} \right)$.
\end{definition}

\begin{theorem}
\label{thm:k_homogeneous}
(1) Bias-free $\krfc$ is homogeneous of order 1. (2) With bias initialized at zero, let $\krb = \krfcb - \krfc$. Then, $\krb$ is homogeneous of order 0.
\end{theorem}


The two kernels $\krfc$ and $\krfcb$ (but not $\krb$) are unbounded. Therefore, their Mercer's representation \eqref{eq:mercer} exists under measures that decay sufficiently fast as $\|\x\| \rightarrow \infty$. Examples include the uniform distribution on the $\|\x\| \le 1$ disk or the standard normal distribution. Such distributions have the virtue of being uniform on all concentric spheres. The following theorem determines the eigenfunctions for these kernels.

\begin{theorem}  \label{thm:eig_outofsphere}
Let $p(r)$ be a decaying density on $[0,\infty)$ such that 
$0 < \int_0^\infty p(r)r^2 dr<\infty$ and $\x,\z\in\Rd$.
\begin{enumerate}
    \item Let $\kr_0(\x,\z)$ be homogeneous of order 1 such that $\kr_0(\x,\z) = \norm{\x} \norm{\z} \hat\kr_0(\frac{\x^T\z}{\norm{\x}\norm{\z}})$. Then its eigenfunctions with respect to $p(\|\x\|)$ are given by $\Psi_{k,j}=a\|\x\|Y_{k,j}\left(\frac{\x}{\|\x\|}\right)$ where $Y_{k,j}$ are the spherical harmonics in $\Spdm$ and $a\in\Real$.
    \item Let $\kr(\x,\z) = \kr_0(\x,\z)$ + $\kr_1(\x,\z)$ so that $\kr_0$ as in 1 and $\kr_1$ is homogeneous of order 0. Then the eigenfunctions of $\kr$ are of the form $\Psi_{k,j} = \left( a \norm{\x} + b \right) Y_{k,j}\left(\frac{\x}{\|\x\|}\right)$.
\end{enumerate}
\end{theorem}

The eigenfunctions of NTK in $\Rd$, therefore, are similar to those in $\Spdm$; they are the spherical harmonics scaled radially in the bias free case, or linearly with the norm when bias is used. With bias, $\krfcb$ has up to $2N(d,k)$ eigenfunctions for every frequency $k$. Compared to the eigenvalues in $\Spdm$, the eigenvalues can change, depending on the radial density $p(r)$, but they maintain their overall asymptotic behavior.

\begin{figure}[tb]
    \centering
    \includegraphics[height=3.5cm]{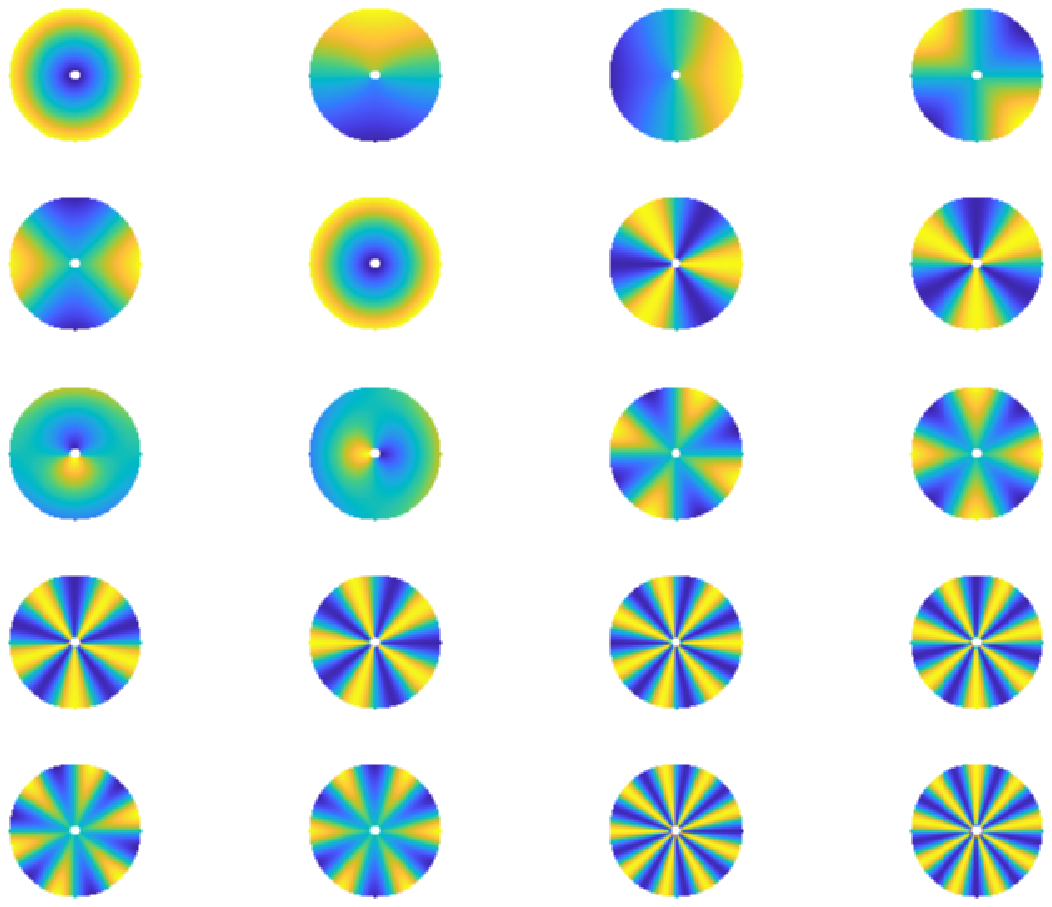}
    \includegraphics[height=3.5cm]{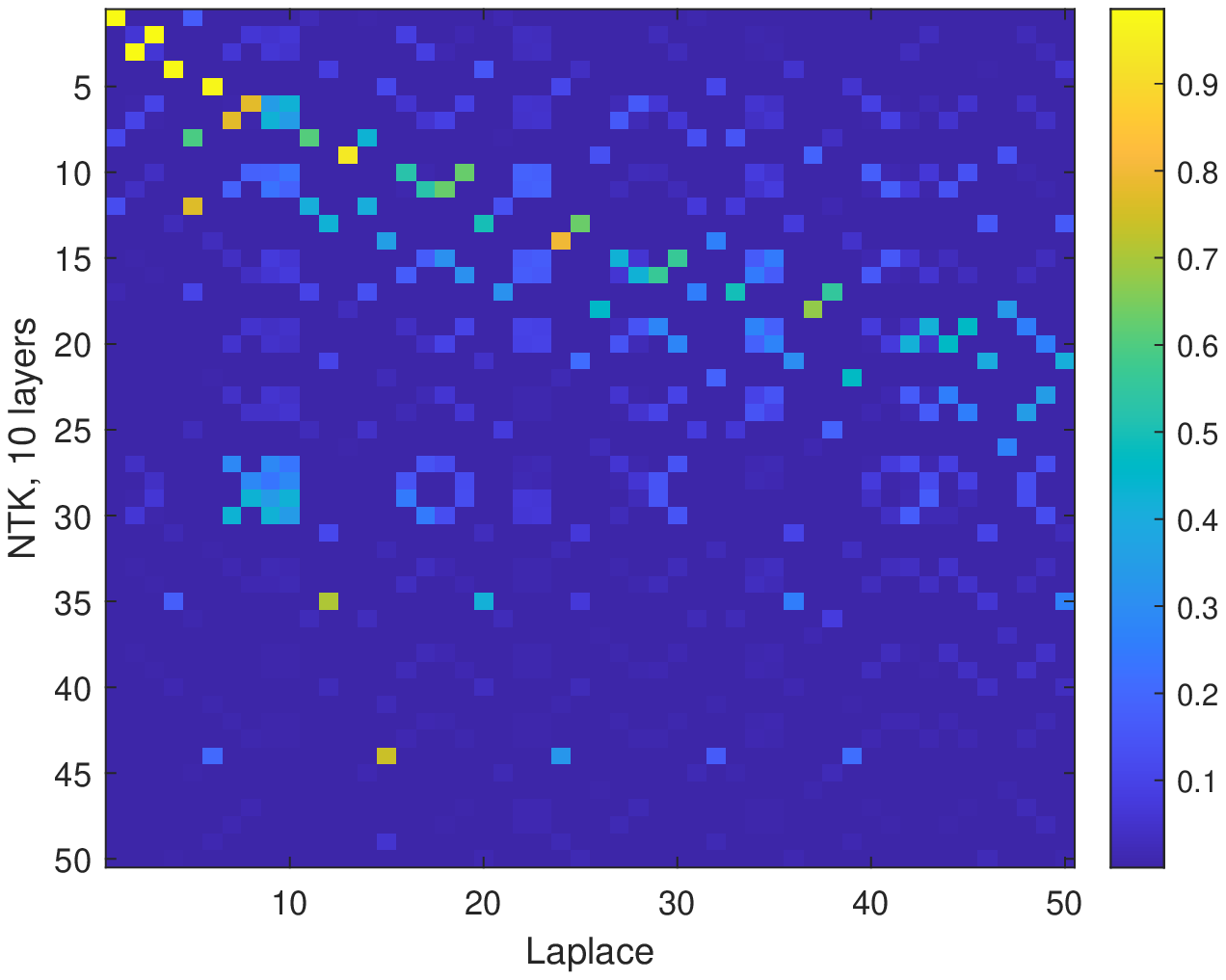}
    \caption{\small Left: plots of the eigenfunctions of NTK for a two layer FC network with bias on the unit disk, arranged in decreasing order of the eigenvalues. The radial shape of the eigenfunctions is evident. For two layers, the eigenvalues of $\krfcl{2}$ are zero for odd $k \ge 3$, while those of $\krbtwo$ are zero for even $k \ge 2$. Therefore we see two ``DC components" (top left and 2nd in 2nd row) and four $k=1$ components (2nd and 3rd in 1st row and 1st and 2nd in third row). The rest of the frequencies are represented twice each. Right: Absolute correlation between the eigenfunctions of NTK and those of the Laplace kernel for data sampled uniformly on the unit disk. It can be seen that eigenfunctions of higher frequency for NTK correlate with eigenfunctions of higher frequency for the Laplace. However, relatively low order components for NTK contain higher frequency components of the Laplace.}
    \label{fig:Rd_eigs}
\end{figure}

In contrast to NTK, the Laplace kernel is shift invariant, and therefore its eigenfunctions are the Fourier transform. The two kernels hence cannot be compared merely by their eigenvalues. Figure~\ref{fig:Rd_eigs} shows the eigenfunctions of NTK along with their correlation to the eigenfunctions of the Laplace kernel. While these differences are large, they seem to make only little difference in experiments, see Section~\ref{sec:experiments} below. It is possible to produce a homogeneous version of the Laplace kernel as follows
\begin{equation}
    \krhexp(\x,\z) = \|\x\|\|\z\| \exp\left(-c\sqrt{1-\frac{\x^T\z}{\|\x\|\|\z\|}}\right).
\end{equation}
Following Thm.~\ref{thm:eig_outofsphere} the eigenfunctions of this kernel are the scaled spherical harmonics and, following Thm.~\ref{thm:exp_decay}, its eigenvalues decay at the rate of $k^{-d}$, much like the NTK.

\resnet{
Finally, we note that NTK for bias-free, residual networks too is homogeneous of order 1, and therefore its eigenfunctions are the scaled spherical harmonics. 
}

\section{Experiments} \label{sec:experiments}

We compare the performance of NTK with Laplace, Gaussian, and $\gamma$-exponential kernels on both small and large scale real datasets. Our goal is to demonstrate: a) Results with the Laplace kernel are quite similar to those obtained by NTK, and b)~The $\gamma$-exponential kernel can achieve slightly better results than NTK. Experimental details are provided in the supplementary material.

\subsection{UCI dataSet}  \label{small_uci}

We compare methods using the same set of 90 small scale UCI datasets (with less than 5000 data points) as in~\cite{Arora2020Harnessing}. The results are provided in Table~\ref{tab:UCI} for the exponential kernels and their homogeneous versions, denoted by the "H-" prefix, as well as for NTK. For completeness, we also cite the results for Random forest (RF), the top classifier identified in~\cite{fernandez2014we}, and neural networks from~\cite{Arora2020Harnessing}. Further comparison of the accuracies obtained with NTK vs. the H-Laplace kernel on each of the 90 datasets is shown in Figure~\ref{fig:scatter_plot}.

We report the same metrics as used in~\cite{Arora2020Harnessing}: Friedman Ranking, Average Accuracy, P90/P95, and PMA. A superior classifier is expected to have lower Friedman rank and higher P90, P95, and PMA. Friedman Ranking~\cite{demvsar2006statistical} reports the average ranking of a given classifier compared to other classifiers. P90/P95 denotes the fraction of datasets on which a classifier achieves
more than $90/95\%$ of the maximum achievable accuracy (i.e., maximum accuracy among all the classifiers~\cite{fernandez2014we}). PMA represents the percentage of maximum accuracy.

From Table~\ref{tab:UCI}, one can observe that the H-Laplace kernel results are the closest to NTK on all the metrics. In fact, as seen in Figure~\ref{fig:scatter_plot}, these methods seem to achieve highly similar accuracies in each of the 90 datasets. Furthermore, the H-$\gamma$-exponential outperforms all the classifiers including NTK on all metrics. Moreover, the homogeneous versions slightly outperform the standard kernels. 
All these methods have hyperparameters that can be optimized.  In~\cite{Arora2020Harnessing}, they search $105$ hyperparameters for NTK.
For a fair comparison, we search for the same number for the $\gamma$-exponential and fewer (70) for the Laplace kernels. 
We note finally that deeper networks yield NTK shapes that are more sharply peaked, corresponding to Laplace kernels with higher values of $c$. This is shown in Fig.~\ref{fig:c_vs_number_of_layers} below.

\begin{figure}[htb]
\begin{floatrow}
\capbtabbox{%
\tiny{{
\begin{tabular}{|l|c|c|c|c|c|}
\toprule
Classifier        & F-Rank  & Average Accuracy             & P90              & P95              & PMA       \\
\midrule

H-$\gamma$-exp.               & \textbf{26.26} & \textbf{82.25\%$\pm$14.07\%} & \textbf{92.22\%} & \textbf{73.33\%} & \textbf{96.07\% $\pm$4.83\%} \\ 

$\gamma$-exp.               & 32.98 & 81.80\%$\pm$14.21\% & 85.56\% & 73.33\% & 95.49\% $\pm$5.31\% \\ 

H-Laplace                & 29.60 & 81.74\%$\pm$13.82\% & 88.89\% & 66.67\% & 95.53\% $\pm$4.84\% \\ 

Laplace                & 33.28 & 81.12\%$\pm$14.16\% & 86.67\% & 65.56\% & 94.88\% $\pm$6.85\% \\ 

H-Gaussian    & 32.66          & 81.46\% $\pm$ 14.83\%        & 84.44\%          & 67.77\% & 94.95\% $\pm$6.25\%          \\ 
Gaussian    & 35.76          & 81.03\% $\pm$ 15.09\%        & 85.56\%          & 72.22\% & 94.56\% $\pm$8.22\%          \\

NTK~\cite{Arora2020Harnessing}              & 28.34 & 81.95\%$\pm$14.10\% & 88.89\% & 72.22\% & 95.72\% $\pm$5.17\% \\

NN~\cite{Arora2020Harnessing}     & 38.06          & 81.02\%$\pm$14.47\%          & 85.56\%          & 60.00\%          & 94.55\% $\pm$5.89\%          \\ 


RF~\cite{Arora2020Harnessing}                & 33.51          & 81.56\% $\pm$13.90\%         & 85.56\%          & 67.78\%          & 95.25\% $\pm$5.30\%          \\ 

\bottomrule
\end{tabular}
}}
}{%
  \caption{\small Performance on the UCI dataset. Lower F-Rank and higher P90, P95, PMA are better numbers.\label{tab:UCI}}%
}

\ffigbox{%
    \begin{minipage}[c]{0.5\textwidth}
        \includegraphics[height=2.9cm]{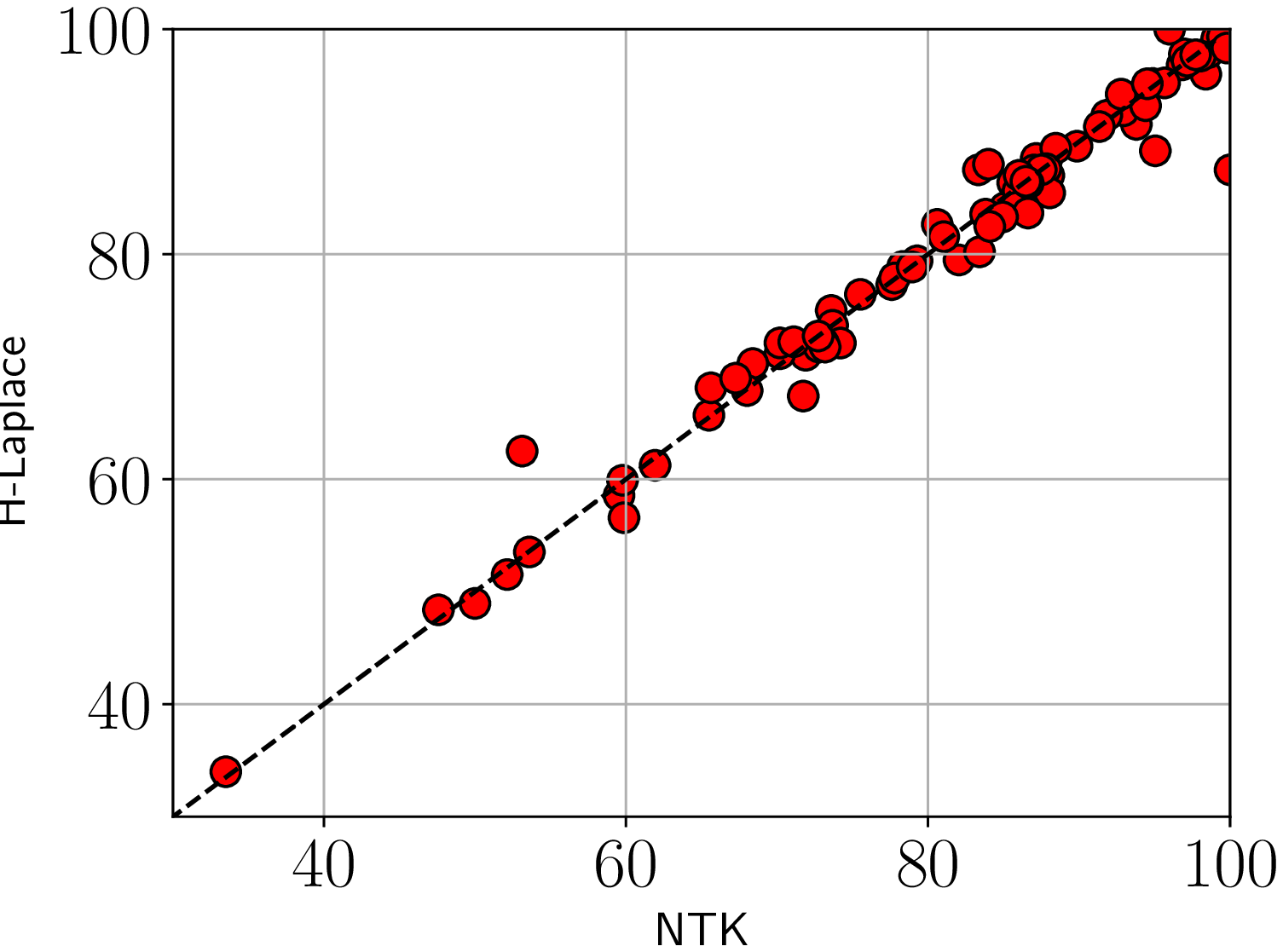}
  \end{minipage}
{%
\caption{\small{Performance comparisons \newline
between NTK and H-Laplace on the
\newline
UCI dataset.} \label{fig:scatter_plot}}
}

}

\end{floatrow}
\end{figure}

\begin{figure}[t]
    \centering
    \includegraphics[height=2.7cm]{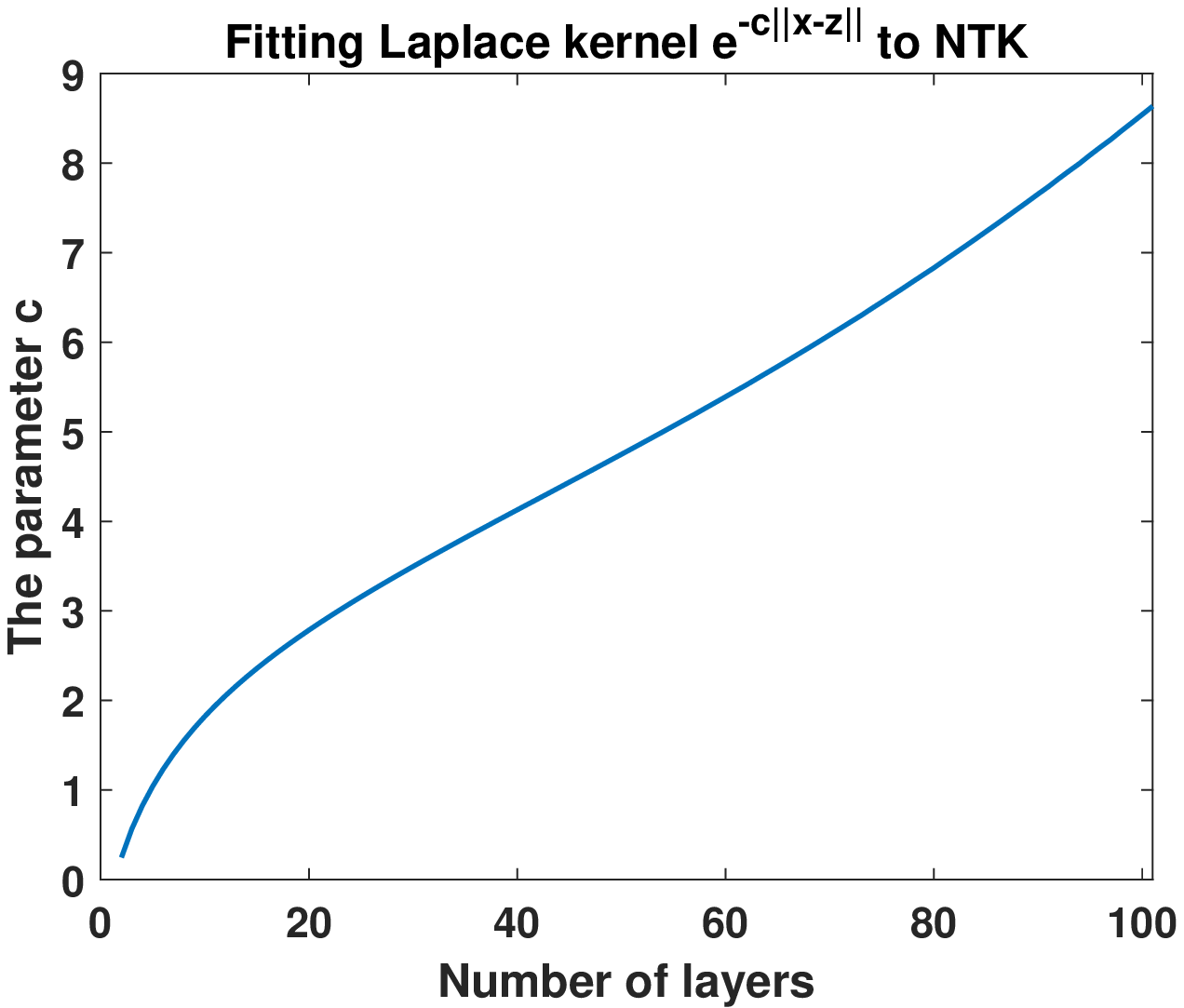}
    \caption{\small Fitting the Laplace kernel to NTK. The graph shows optimal width ($c$) of the Laplace kernel that is fitted to NTK with different number of layers.}
    \label{fig:c_vs_number_of_layers}
\end{figure}

\subsection{Large scale datasets}\label{large_uci}

We leverage FALKON~\cite{rudi2017falkon}, an efficient approximate kernel method to conduct large scale regression and classification tasks following the setup of~\cite{rudi2017falkon}. The results and datasets details are reported in Table~\ref{table:largeData}. We searched for hyperparameters based on a small validation dataset for all the methods and used the standard train/test partition provided on the UCI repository. From Table~\ref{table:largeData}, one can notice that NTK and H-Laplace perform similarly. For each dataset, either the $\gamma$-exponential or Gaussian kernels slightly outperforms these two kernels.

\begin{table*}[!htbp]
  \centering
  \begin{tiny}
  \begin{tabular}{|l|c|c|c|}
    \toprule
        & MillionSongs~\cite{bertin2011million} & SUSY~\cite{baldi2014searching} & HIGGS~\cite{baldi2014searching} \\
        \midrule
            \#Training Data                      &  $4.6\times10^5$    & $5\times10^6$ & $1.1\times 10^7$ \\
            \#Features                      & 90    & 18 & 28 \\
            Problem Type                      & Regression    & Classification & Classification \\
                        Performance Metric                      & MSE   & AUC & AUC \\
    \hline
    \hline

    H-$\gamma$-exp.                      & \textbf{78.6417}    & 87.686 & \textbf{82.281} \\
    
    H-Laplace                      & 79.7941    & 87.670 & 81.995 \\
    
    NTK                    & 79.9666    & 87.673 & 82.089 \\
    
    H-Gaussian                     & 79.6255    & \textbf{87.689} & 81.967 \\
    
    \hline

    Neural Network \cite{baldi2014searching}  & - & 87.500    & 81.600 \\
    
    Deep Neural Network \cite{baldi2014searching} & -  & \textbf{87.900}$^{*}$    & \textbf{88.500}$^{*}$   \\
    
    \bottomrule
  \end{tabular} 
     \caption{\small Performance on the large scale datasets. We report MSE~(lower is better) for the regression problem, and AUC~(higher is better) for the classification problems.} 
     \label{table:largeData}
\end{tiny}
\end{table*}

\subsection{Hierarchical convolutional kernels}

Convolutional NTKs (CNTK) were shown to express the limit of convolutional neural networks when the number of channels tends to infinity, and recent empirical results showed that the two achieve similar accuracy on test data \cite{arora2019exact,li2019enhanced}.  CNTK is defined roughly by recursively applying NTK to image patches. For our final experiment we constructed alternative hierarchical kernels, in the spirit of \cite{Bo2010nips,Mairal2014conv}, by recursively applying exponential kernels in a manner similar to CNTK. The new kernel, denoted C-Exp, is applied first to pairs of $3 \times 3$ image patches, then to $3 \times 3$ patches of kernel values, and so forth. A detailed algorithm is provided in the supplementary material.
We applied the kernel (using the homogeneous versions of the Laplace, Gaussian and $\gamma$-exponential kernels) to the Cifar-10 dataset and compared it to CNTK. Our experimental conditions and results for CNTK are identical to those of \cite{arora2019exact}. (Note that these do not include global average pooling.) Consistent with our previous experiments, Table~\ref{tab:cexp} shows that these kernels are on par with the CNTK with small advantage to the $\gamma$-exponential kernel. This demonstrates that the four kernels maintain similar performance even after repeated application.

\begin{table}[h]

\tiny
\centering
\begin{tabular}{|l|c|c|}
 \toprule
  Method& Accuracy (50k) &Accuracy(2k) \\
 \midrule
 CNTK   & 66.4\%   &43.9\%\\
 C-Exp Laplace &   65.2\%  & 44.2\%   \\
 C-Exp $\gamma$-exponential & \textbf{67.0\%}  &  \textbf{45.2}\%\\
 C-Exp Gaussian  &66.8\% & 45.0\%\\
 
 \bottomrule
\end{tabular}\

\caption{\label{tab:cexp}
\small Classification accuracy for the CIFAR-10 dataset for our C-Exp hierarchical kernels, compared to CNTK. The two columns show results with training on the full dataset and on the first 2000 examples.}
\end{table}

\cmt{
\begin{table}

\tiny
\centering
\begin{tabular}{|l||c|c|}
 \hline
  Method& Accuracy (50k) &Accuracy(2k) \\
 \hline
 CNTK   & 66.4\%   &43.9\%\\
 C-Exp Laplace &   65.2\%  & 44.2\%   \\
 C-Exp $\gamma$-exponential & \textbf{67.0\%}  &  \textbf{45.2}\%\\
 C-Exp Gaussian  &66.8\% & 45.0\%\\
 
 \hline
\end{tabular}\

\caption{\label{tab:cexp}
\small Classification accuracy for the CIFAR-10 dataset for our C-Exp hierarchical kernels, compared to CNTK. The two columns show results on the full dtaset and on the first 2000 examples.}
\end{table}
}

\section{Conclusions}
Our paper has considered the relationship between NTK and the classic Laplace kernel.  Our main result is to show that for data normalized on the unit hypersphere, these two kernels have the same RKHS.  Experiments show that the two kernels perform almost identically on a wide range of real-world applications.  Coupled with prior results that show that kernel methods using NTK mimic the behavior of FC neural networks, our results suggest that much insight about neural networks can be obtained from analysis of the well-known Laplace kernel, which has a simple closed form.

Neural networks do offer great flexibility not easily translated to kernel methods.  They can naturally be applied to large data sets, and  researchers have developed many techniques for training them, such as dropout and batch normalization, that improve performance but do not directly translate to kernel methods.  Furthermore, while we study feed-forward fully connected networks, analyzing more complex architectures, such as CNNs, GANs, autoencoders and recurrent networks, or the effect of other activation functions, remains a significant challenge for future work. Comparing the eigenfunctions of NTK with those of classical kernels under non-uniform distribution is yet a further challenge.

\section*{Broader impact}
This work explains the success of deep, fully connected networks through their similarity to exponential kernels. Such an analysis may allow for a better interpretability of deep network models.

\section*{Acknowledgements}

The authors thank the U.S.- Israel Binational Science Foundation, grant number 2018680, the National Science Foundation, grant no. IIS-1910132, the Quantifying Ensemble Diversity for Robust Machine Learning (QED for RML) program from DARPA and the Guaranteeing AI Robustness Against Deception (GARD) program from DARPA for their support of this project.


{\small
\bibliographystyle{ieee_fullname}
\bibliography{citations}

\begin{thebibliography}{10}\itemsep=-1pt

\bibitem{allen-zhu2019}
Zeyuan Allen-Zhu, Yuanzhi Li, and Zhao Song.
\newblock A convergence theory for deep learning via over-parameterization.
\newblock In Kamalika Chaudhuri and Ruslan Salakhutdinov, editors, {\em
  Proceedings of the 36th International Conference on Machine Learning},
  volume~97, pages 242--252, 2019.

\bibitem{arora2019exact}
Sanjeev Arora, Simon~S Du, Wei Hu, Zhiyuan Li, Russ~R Salakhutdinov, and
  Ruosong Wang.
\newblock On exact computation with an infinitely wide neural net.
\newblock In {\em Advances in Neural Information Processing Systems}, pages
  8139--8148, 2019.

\bibitem{arora2019fine}
Sanjeev Arora, Simon~S Du, Wei Hu, Zhiyuan Li, and Ruosong Wang.
\newblock Fine-grained analysis of optimization and generalization for
  overparameterized two-layer neural networks.
\newblock {\em arXiv preprint arXiv:1901.08584}, 2019.

\bibitem{Arora2020Harnessing}
Sanjeev Arora, Simon~S. Du, Zhiyuan Li, Ruslan Salakhutdinov, Ruosong Wang, and
  Dingli Yu.
\newblock Harnessing the power of infinitely wide deep nets on small-data
  tasks.
\newblock In {\em International Conference on Learning Representations}, 2020.

\bibitem{bach2017breaking}
Francis Bach.
\newblock Breaking the curse of dimensionality with convex neural networks.
\newblock {\em The Journal of Machine Learning Research}, 18(1):629--681, 2017.

\bibitem{basriNonuniform}
Ronen Basri, Meirav Galun, Amnon Geifman, David Jacobs, Yoni Kasten, and Shira
  Kritchman.
\newblock Frequency bias in neural networks for input of non-uniform density.
\newblock In {\em International Conference on Machine Learning}, 2020.

\bibitem{ronen2019convergence}
Ronen Basri, David Jacobs, Yoni Kasten, and Shira Kritchman.
\newblock The convergence rate of neural networks for learned functions of
  different frequencies.
\newblock In {\em Advances in Neural Information Processing Systems}, pages
  4763--4772, 2019.

\bibitem{belkin2018overfitting}
Mikhail Belkin, Daniel~J Hsu, and Partha Mitra.
\newblock Overfitting or perfect fitting? risk bounds for classification and
  regression rules that interpolate.
\newblock In {\em Advances in neural information processing systems}, pages
  2300--2311, 2018.

\bibitem{belkin2018understand}
Mikhail Belkin, Siyuan Ma, and Soumik Mandal.
\newblock To understand deep learning we need to understand kernel learning.
\newblock {\em arXiv preprint arXiv:1802.01396}, 2018.

\bibitem{bertin2011million}
Thierry Bertin-Mahieux, Daniel~P.W. Ellis, Brian Whitman, and Paul Lamere.
\newblock The million song dataset.
\newblock In {\em {Proceedings of the 12th International Conference on Music
  Information Retrieval ({ISMIR})}}, 2011.

\bibitem{bietti2019inductive}
Alberto Bietti and Julien Mairal.
\newblock On the inductive bias of neural tangent kernels.
\newblock In {\em Advances in Neural Information Processing Systems}, pages
  12873--12884, 2019.

\bibitem{Bo2010nips}
Liefeng Bo, Xiaofeng Ren, and Dieter Fox.
\newblock Kernel descriptors for visual recognition.
\newblock In J.~D. Lafferty, C.~K.~I. Williams, J. Shawe-Taylor, R.~S. Zemel,
  and A. Culotta, editors, {\em Advances in Neural Information Processing
  Systems 23}, pages 244--252. 2010.

\bibitem{bordelon2020spectrum}
Blake Bordelon, Abdulkadir Canatar, and Cengiz Pehlevan.
\newblock Spectrum dependent learning curves in kernel regression and wide
  neural networks.
\newblock {\em arXiv preprint arXiv:2002.02561}, 2020.

\bibitem{cao2019towards}
Yuan Cao, Zhiying Fang, Yue Wu, Ding-Xuan Zhou, and Quanquan Gu.
\newblock Towards understanding the spectral bias of deep learning.
\newblock {\em arXiv preprint arXiv:1912.01198}, 2019.

\bibitem{chapelle2007training}
Olivier Chapelle.
\newblock Training a support vector machine in the primal.
\newblock {\em Neural computation}, 19(5):1155--1178, 2007.

\bibitem{chen2020deep}
Lin Chen and Sheng Xu.
\newblock Deep neural tangent kernel and laplace kernel have the same {RKHS}.
\newblock {\em arXiv preprint arXiv:2009.10683}, 2020.

\bibitem{cho2009kernel}
Youngmin Cho and Lawrence~K Saul.
\newblock Kernel methods for deep learning.
\newblock In {\em Advances in neural information processing systems}, pages
  342--350, 2009.

\bibitem{cho2011analysis}
Youngmin Cho and Lawrence~K Saul.
\newblock Analysis and extension of arc-cosine kernels for large margin
  classification.
\newblock {\em arXiv preprint arXiv:1112.3712}, 2011.

\bibitem{dai2014scalable}
Bo Dai, Bo Xie, Niao He, Yingyu Liang, Anant Raj, Maria-Florina~F Balcan, and
  Le Song.
\newblock Scalable kernel methods via doubly stochastic gradients.
\newblock In {\em Advances in Neural Information Processing Systems}, pages
  3041--3049, 2014.

\bibitem{daniely2016toward}
Amit Daniely, Roy Frostig, and Yoram Singer.
\newblock Toward deeper understanding of neural networks: The power of
  initialization and a dual view on expressivity.
\newblock In {\em Advances In Neural Information Processing Systems}, pages
  2253--2261, 2016.

\bibitem{demvsar2006statistical}
Janez Dem{\v{s}}ar.
\newblock Statistical comparisons of classifiers over multiple data sets.
\newblock {\em Journal of Machine learning research}, 7(Jan):1--30, 2006.

\bibitem{fan2020spectra}
Zhou Fan and Zhichao Wang.
\newblock Spectra of the conjugate kernel and neural tangent kernel for
  linear-width neural networks.
\newblock {\em arXiv preprint arXiv:2005.11879}, 2020.

\bibitem{fernandez2014we}
Manuel Fern{\'a}ndez-Delgado, Eva Cernadas, Sen{\'e}n Barro, and Dinani Amorim.
\newblock Do we need hundreds of classifiers to solve real world classification
  problems?
\newblock {\em The journal of machine learning research}, 15(1):3133--3181,
  2014.

\bibitem{gallier2009notes}
Jean Gallier.
\newblock Notes on spherical harmonics and linear representations of lie
  groups.
\newblock {\em preprint}, 2009.

\bibitem{Genton2001}
Marc~G. Genton.
\newblock Classes of kernels for machine learning: A statistics perspective.
\newblock {\em Journal of Machine Learning Research}, 2:299--312, 2001.

\bibitem{ghorbani2019linearized}
Behrooz Ghorbani, Song Mei, Theodor Misiakiewicz, and Andrea Montanari.
\newblock Linearized two-layers neural networks in high dimension.
\newblock {\em arXiv preprint arXiv:1904.12191}, 2019.

\bibitem{huang2020deep}
Kaixuan Huang, Yuqing Wang, Molei Tao, and Tuo Zhao.
\newblock Why do deep residual networks generalize better than deep feedforward
  networks?--a neural tangent kernel perspective.
\newblock {\em arXiv preprint arXiv:2002.06262}, 2020.

\bibitem{jacot2018neural}
Arthur Jacot, Franck Gabriel, and Cl{\'e}ment Hongler.
\newblock Neural tangent kernel: Convergence and generalization in neural
  networks.
\newblock In {\em Advances in neural information processing systems}, pages
  8571--8580, 2018.

\bibitem{jetter1999error}
Kurt Jetter, Joachim St{\~A}{\c{k}}ckler, and Joseph Ward.
\newblock Error estimates for scattered data interpolation on spheres.
\newblock {\em Mathematics of Computation}, 68(226):733--747, 1999.

\bibitem{kanagawa2018gaussian}
Motonobu Kanagawa, Philipp Hennig, Dino Sejdinovic, and Bharath~K
  Sriperumbudur.
\newblock Gaussian processes and kernel methods: A review on connections and
  equivalences.
\newblock {\em arXiv preprint arXiv:1807.02582}, 2018.

\bibitem{kimeldorf1970correspondence}
George~S Kimeldorf and Grace Wahba.
\newblock A correspondence between bayesian estimation on stochastic processes
  and smoothing by splines.
\newblock {\em The Annals of Mathematical Statistics}, 41(2):495--502, 1970.

\bibitem{lee2017deep}
Jaehoon Lee, Yasaman Bahri, Roman Novak, Samuel~S Schoenholz, Jeffrey
  Pennington, and Jascha Sohl-Dickstein.
\newblock Deep neural networks as gaussian processes.
\newblock {\em arXiv preprint arXiv:1711.00165}, 2017.

\bibitem{li2019enhanced}
Zhiyuan Li, Ruosong Wang, Dingli Yu, Simon~S Du, Wei Hu, Ruslan Salakhutdinov,
  and Sanjeev Arora.
\newblock Enhanced convolutional neural tangent kernels.
\newblock {\em arXiv preprint arXiv:1911.00809}, 2019.

\bibitem{liang2018just}
Tengyuan Liang and Alexander Rakhlin.
\newblock Just interpolate: Kernel" ridgeless" regression can generalize.
\newblock {\em arXiv preprint arXiv:1808.00387}, 2018.

\bibitem{Mairal2014conv}
Julien Mairal, Piotr Koniusz, Zaid Harchaoui, and Cordelia Schmid.
\newblock Convolutional kernel networks.
\newblock In Z. Ghahramani, M. Welling, C. Cortes, N.~D. Lawrence, and K.~Q.
  Weinberger, editors, {\em Advances in Neural Information Processing Systems
  27}, pages 2627--2635. 2014.

\bibitem{matthews2018gaussian}
Alexander G de~G Matthews, Mark Rowland, Jiri Hron, Richard~E Turner, and
  Zoubin Ghahramani.
\newblock Gaussian process behaviour in wide deep neural networks.
\newblock {\em arXiv preprint arXiv:1804.11271}, 2018.

\bibitem{micchelli1979design}
Charles~A Micchelli and Grace Wahba.
\newblock Design problems for optimal surface interpolation.
\newblock Technical report, Wisconsin University Madison, Dept. of Statistics,
  1979.

\bibitem{minh2006mercer}
Ha~Quang Minh, Partha Niyogi, and Yuan Yao.
\newblock Mercer’s theorem, feature maps, and smoothing.
\newblock In {\em International Conference on Computational Learning Theory},
  pages 154--168. Springer, 2006.

\bibitem{narcowich2007approximation}
Francis~J Narcowich, Xinping Sun, and Joseph~D Ward.
\newblock Approximation power of rbfs and their associated sbfs: a connection.
\newblock {\em Advances in Computational Mathematics}, 27(1):107--124, 2007.

\bibitem{narcowich2002scattered}
Francis~J Narcowich and Joseph~D Ward.
\newblock Scattered data interpolation on spheres: error estimates and locally
  supported basis functions.
\newblock {\em SIAM Journal on Mathematical Analysis}, 33(6):1393--1410, 2002.

\bibitem{neal2012bayesian}
Radford~M Neal.
\newblock {\em Bayesian learning for neural networks}, volume 118.
\newblock Springer Science \& Business Media, 2012.

\bibitem{novak2018bayesian}
Roman Novak, Lechao Xiao, Jaehoon Lee, Yasaman Bahri, Greg Yang, Jiri Hron,
  Daniel~A Abolafia, Jeffrey Pennington, and Jascha Sohl-Dickstein.
\newblock Bayesian deep convolutional networks with many channels are gaussian
  processes.
\newblock {\em arXiv preprint arXiv:1810.05148}, 2018.

\bibitem{GammaExponential}
Carl~Edward Rasmussen and Christopher K.~I. Williams.
\newblock {\em Gaussian Processes for Machine Learning}.
\newblock MIT Press, 2006.

\bibitem{rudi2017falkon}
Alessandro Rudi, Luigi Carratino, and Lorenzo Rosasco.
\newblock Falkon: An optimal large scale kernel method.
\newblock In {\em Advances in Neural Information Processing Systems}, pages
  3888--3898, 2017.

\bibitem{baldi2014searching}
Pierre Baldi~Peter Sadowski and Daniel Whiteson.
\newblock Searching for exotic particles in high-energy physics with deep
  learning.
\newblock {\em Nature communications}, 5, 2014.

\bibitem{saitoh2016theory}
Saburou Saitoh and Yoshihiro Sawano.
\newblock {\em Theory of reproducing kernels and applications}.
\newblock Springer, 2016.

\bibitem{scholkopf2001learning}
Bernhard Scholkopf and Alexander~J Smola.
\newblock {\em Learning with kernels: support vector machines, regularization,
  optimization, and beyond}.
\newblock MIT press, 2001.

\bibitem{stein2016introduction}
Elias~M Stein and Guido Weiss.
\newblock {\em Introduction to Fourier analysis on Euclidean spaces (PMS-32)},
  volume~32.
\newblock Princeton university press, 2016.

\bibitem{watson1995treatise}
George~Neville Watson.
\newblock {\em A treatise on the theory of Bessel functions}.
\newblock Cambridge university press, 1966.

\bibitem{williams1997computing}
Christopher~KI Williams.
\newblock Computing with infinite networks.
\newblock In {\em Advances in neural information processing systems}, pages
  295--301, 1997.

\bibitem{yang2019fine}
Greg Yang and Hadi Salman.
\newblock A fine-grained spectral perspective on neural networks.
\newblock {\em arXiv preprint arXiv:1907.10599}, 2019.

\end{thebibliography}
}

\clearpage

\appendix

\section{Formulas for NTK}
\label{app:NTK_formulas}

We begin by providing the recursive definition of NTK for fully connected (FC) networks with bias initialized at zero. The formulation includes a parameter $\beta$ that when set to zero the recursive formula coincides with the formula given in \cite{arora2019exact} for  bias-free networks.

\paragraph{The network model.} We  consider a $L$-hidden-layer fully-connected neural network (in total $L+1$ layers) with bias. Let $\x \in \Real^{d}$ (and denote $d_0=d$), we assume each layer $l \in [L]$ of hidden units includes $d_l$ units. The network model is expressed as 
\begin{align*}
    \g^{(0)}(\x) & = \x  \\
    \f^{(l)}(\x) & = W^{(l)}\g^{(l-1)} (\x) + \beta \mathbf{b}^{(l)}  \in \Real^{d_l}, ~~~~~l=1,\ldots L \\
    \g^{(l)}(\x) & = \sqrt{\frac{c_\sigma}{d_l}}
    \sigma\left(\f^{(l)}(\x)\right)\in \Real^{d_l}, ~~~~~l=1,\ldots L \\
    f(\theta,\x) & = f^{(L+1)}(\x) = W^{(L+1)} \cdot \g^{(L)}(\x) + \beta b^{(L+1)}
\end{align*}
The network parameters $\theta$ include $W^{(L+1)},W^{(L)},...,W^{(1)}$, where $W^{(l)} \in \Real^{d_l \times d_{l-1}}$, $\mathbf{b}^{(l)}\in \Real^{d_l\times 1}$, $W^{(L+1)} \in \Real^{1 \times d_L}$, $b^{(L+1)}\in \Real$, $\sigma$ is the activation function  and $c_{\sigma} = 1/\left( \mathbb{E}_{z \sim \mathcal{N}(0,1)} [\sigma(z)^2] \right) $.  The  network parameters are initialized with ${\cal N}(0,I)$, except for the biases $\{\mathbf{b}^{(1)}, \ldots,\mathbf{b}^{(L)},b^{(L+1)}\}$, which are initialized with zero.

\paragraph{The recursive formula for NTK.}
The recursive formula in \cite{jacot2018neural} assumes the bias is initialized with a normal distribution. Here we assume the bias is initialized at zero, yielding a sightly different formulation, which can be readily derived from \cite{jacot2018neural}'s formulation.

Given $\x,\z \in \Real^d$, we denote the NTK for this fully connected network with bias by $\krfcbl{L+1} (\x,\z) := \Theta^{(L)}(\x,\z)$. The kernel $\Theta^{(L)}(\x,\z)$ is defined using the following recursive definition. Let $h \in [L]$ then
\begin{equation}  \label{eq:ntkdeep}
    \Theta^{(h)}(\x,\z) = \Theta^{(h-1)}(\x,\z) \dot{\Sigma}^{(h)}(\x,\z) + \Sigma^{(h)}(\x,\z)+\beta^2,
\end{equation}
where
\begin{equation*}
    \Sigma^{(0)}(\x,\z) = \x^T\z
\end{equation*}
\begin{equation*}
    \Theta^{(0)}(\x,\z)=\Sigma^{(0)}(\x,\z)+\beta^2.
\end{equation*}
and we define
\begin{align*}
    \Sigma^{(h)}(\x,\z) &= c_\sigma\mathbb{E}_{(u,v)\backsim N(0,\Lambda^{(h-1)})}\left(\sigma(u)\sigma(v)\right)\\
    \dot \Sigma^{(h)}(\x,\z) &= c_\sigma\mathbb{E}_{(u,v)\backsim N(0,\Lambda^{(h-1)})}\left( \dot \sigma(u) \dot \sigma(v)\right)\\
    \Lambda^{(h-1)} &= \begin{pmatrix}
    \Sigma^{(h-1)}(\x,\x) & \Sigma^{(h-1)}(\x,\z) \\
    \Sigma^{(h-1)}(\z,\x)& \Sigma^{(h-1)}(\z,\z)
    \end{pmatrix}.
\end{align*}

Now, let
\begin{equation}  \label{eq:lambda}
 \lambda^{(h-1)}(\x,\z) = \frac{\Sigma^{(h-1)}(\x,\z)}{\sqrt{\Sigma^{(h-1)}(\x,\x)\Sigma^{(h-1)}(\z,\z)}}.
\end{equation}
By definition $|\lambda^{(h-1)}|\leq 1$, and for ReLU activation we have $c_\sigma=2$ and
\begin{align} \label{eq:sigma}
    \Sigma^{(h)}(\x,\z) &= c_\sigma \frac{\lambda^{(h-1)} (\pi-\arccos(\lambda^{(h-1)}))+\sqrt{1-(\lambda^{(h-1)})^2}}{2\pi}\sqrt{\Sigma^{(h-1)}(\x,\x)\Sigma^{(h-1)}(\z,\z)} \\
    \label{eq:sigmap}
    \dot \Sigma^{(h)}(\x,\z) &= c_\sigma \frac{ \pi-\arccos(\lambda^{(h-1)})}{2\pi}.
\end{align}
The parameter $\beta$ allows us to consider a fully-connected network either with ($\beta > 0$) or without bias ($\beta = 0$). When $\beta =0$, the recursive formulation is the same as existing derivations, e.g., \cite{jacot2018neural}.
Finally, the normalized NTK of a FC network with $L+1$ layers, without bias, is given by 
$\frac{1}{L+1} \krfcl{L+1}  (\x_i,\x_j)$.

\paragraph{NTK for a two-layer FC network on $\Spdm$.}
Using the recursive formulation above, for points on the hypersphere $\Spdm$ NTK for a two-layer FC network with bias initialized at 0, is as follows.
Let $u=\x^T\z$, with $\x,\z \in \Spdm$. Then, 
\begin{align*}
    \krfctwo(\x,\z)&=\Theta^{(1)}(\x,\z) \\
    &=\Theta^{(0)}(\x,\z) \dot{\Sigma}^{(1)}(\x,\z)+\Sigma^{(1)}(\x,\z)+\beta^2 \\
    &=(u+\beta^2)\frac{\pi-\arccos(u)}{\pi} + \frac{u(\pi-\arccos(u))+\sqrt{1-u^2}}{\pi}+\beta^2.
\end{align*}
Rearranging, we get
\begin{equation}
\label{eq:app:ntk_two_layers}
    \krfctwo(\x,\z)=\krfctwo(u) = \frac{1}{\pi} \left( (2u+\beta^2)(\pi-\arccos(u)) + \sqrt{1-u^2} \right)+\beta^2.
\end{equation}

\section{NTK on $\Spdm$}
\label{app:NTK_on_sphere}

This section provides a characterization of NTK on the hypersphere $\Spdm$ under the uniform measure.
The recursive formulas of the kernels are given in Appendix \ref{app:NTK_formulas}.

\begin{lemma}
\label{lemma:app:ntk_on_sphere_zonal}
    Let $\krfcb(\x,\z)$, $\x,\z \in \Spdm$, denote the NTK kernels for FC  networks with $L \ge 2$ layers, possibly with bias initialized with zero. This kernel is zonal, i.e., $\krfcb(\x,\z)=\krfcb(\x^T\z)$. 
\end{lemma}

\begin{proof}
See Appendix \ref{Appendix:NTK_Rd}. 

\end{proof}

To prove the next theorem, we recall several results on the the arithmetics of RKHS, following \cite{daniely2016toward,saitoh2016theory}.

\subsection {RKHS for sums and products of kernels.}
\label{sec:app:preliminaries}
Let $\kr_1, \kr_2: \X \times \X \rightarrow \Real$ be kernels with RKHS 
 $\hh_{\kr_1}$ and $\hh_{\kr_2}$, respectively. Then,
\begin{enumerate}
\item 
{\bf Aronszajn's kernel sum theorem.} The RKHS for $\kr=\kr_1+\kr_2$ is given by  
$\hh_{\kr_1+\kr_2} =\{f_1+f_2~|~f_1 \in \hh_{\kr_1},~f_2 \in \hh_{\kr_2}\}$

\item 
This yields the {\bf kernel sum inclusion.} $\hh_{\kr_1},\hh_{\kr_2}\subseteq \hh_{\kr_1+\kr_2}$

\item {\bf Norm addition inequality.}
$\norm{f_1+f_2}_{\hh_{\kr_1+\kr_2}} \leq \norm{f_1}_{\hh_{\kr_1}} + \norm{f_2}_{\hh_{\kr_2}}
$

\item {\bf Norm product inequality.}
$\norm{f_1 \cdot f_2}_{\hh_{\kr_1 \cdot \kr_2}} \leq \norm{f_1}_{\hh_{\kr_1}} \cdot \norm{f_2}_{\hh_{\kr_2}}$

\item {\bf Aronszajn's inclusion theorem.} $\hh_{\kr_1} \subseteq \hh_{\kr_2}$ if and only if $\exists s>0$, such that $\kr_1 \ll s^2 \kr_2$, where the latter notation means that $s^2 \kr_2 - \kr_1$ is a positive definite kernel over $\X$. 
\end{enumerate}

\subsection{The decay rate of the eigenvalues of NTK}
\begin{theorem}
\label{app:thm:NTKdecay}
Let $\x,\z \in \mathbb{S}^{d-1}$. With bias initialized at zero and $\beta>0$:
\begin{enumerate}
\item $\krfcb$ can be decomposed according to 
\begin{equation}  
\label{app:eq:sh}
    \krfcb(\x,\z) = \sum_{k=0}^\infty \lambda_k \sum_{j=1}^{N(d,k)} Y_{k,j}(\x) Y_{k,j}(\z),
\end{equation} with $\lambda_k>0$  for all $k \ge 0$ and into $Y_{k,j}$ are the spherical harmonics of $\Spdm$, and
\item $\exists k_0$ and constants $C_1,C_2,C_3>0$ that depend on the dimension $d$ such that $\forall k>k_0$
\begin{enumerate}
    \item $C_1 k^{-d} \leq \lambda_k \leq C_2 k^{-d}$ if $L=2$, and 
    \item $C_3 k^{-d} \leq \lambda_k$ if $L \ge 3$.
\end{enumerate}  
\end{enumerate}
\end{theorem}

We split the theorem into the next two lemmas. The first lemma handles NTK of two-layer FC networks with bias, and the second lemma handles NTK for deep networks.

\begin{lemma}
Let $\x,\z \in \mathbb{S}^{d-1}$ and $\krfctwo(\x^T\z)$ as defined in  \eqref{eq:app:ntk_two_layers} with $\beta>0$. Then, $\krfctwo$ decomposes according to \eqref{app:eq:sh} where $\lambda_k>0$ for all $k \ge 0$ and $\exists k_0$ such that $\forall k \geq k_0$
\begin{equation*}
    C_1 k^{-d} \leq \lambda_k \leq C_2 k^{-d},
\end{equation*}
where $C_1,C_2>0$ are constants that depend on the dimension $d$.
\end{lemma} 

\begin{proof}
To prove the lemma we leverage the results of \cite{bach2017breaking,bietti2019inductive}. 
First, under the assumption of the uniform measure on $\Spdm$, we can apply Mercer decomposition to $\krfctwo(\x,\z)$, where the eigenfunctions are the spherical harmonics. This is due to the observation that  $\krfctwo(\x,\z)$ is   positive and zonal in $\Spdm$. It is zonal by Lemma \ref{lemma:app:ntk_on_sphere_zonal} and  positive, since $\krfctwo$ can be decomposed as 
\begin{align*}
    \krfctwo(u) &= \frac{1}{\pi} \left( (2u+\beta^2)(\pi-\arccos(u)) + \sqrt{1-u^2} \right)+\beta^2 \\
   &= \frac{1}{\pi} \left( 2u(\pi-\arccos(u)) + \sqrt{1-u^2} \right) + \frac{1}{\pi}\beta^2\left(\pi-\arccos(u)\right) + \beta^2 \\
   & := \kappa(\x^T \z) +  \beta^2\kappa_0(\x^T \z) +\beta^2,
\end{align*}
where $\kappa(\x^T \z)$ is the NTK for a bias-free, two-layer network introduced in \cite{bietti2019inductive} and $\kappa_0(\x^T \z)$ is known to be the zero-order arc-cosine kernel \cite{cho2011analysis}. By kernel arithmetic, this yields another kernel and this means that $\krfctwo$ is a  positive kernel.

Furthermore, according to Proposition 5 in \cite{bietti2019inductive}
$$\kappa(\x^T \z)= \sum_{k=0}^{\infty} \mu_k \sum_{j=1}^{N(d,k)} Y_{k,j}(\x) Y_{k,j}(\z),$$
where $Y_{k,j}, j=1,\ldots,N(d,k)$ are spherical harmonics of degree $k$, and the eigenvalues $\mu_k$ satisfy $\mu_0, \mu_1 >0$, $\mu_k=0$ if $k=2j+1$ with $j \geq 1$ and otherwise, $\mu_k>0$ and $\mu_k \sim C(d) k^{-d}$ as $k \rightarrow \infty$, with $C(d)$ a constant depending only on $d$. Next, following Lemma 17 in \cite{bietti2019inductive}
the eigenvalues of $\kappa_0(\x^T \z)$, denoted $\eta_k$ satisfy  $\eta_0, \eta_1 > 0$, $\eta_k > 0$ if $k=2j+1$, with $j \geq 1$ and behave asymptotically as $C_0(d) k^{-d}$. 
Consequently, $\krfctwo = \kappa+ \beta^2 \kappa_0 + \beta^2$, and since both $\kappa$ and $\kappa_0$ have the spherical harmonics as their eigenfunctions, their eigenvalues are given by $\lambda_k=\mu_k+\beta^2\eta_k>0$ for $k>0$ and $\lambda_0=\mu_0+\beta^2\eta_0+\beta^2>0$, and asymptotically $\lambda_k \sim {\tilde C}(d) k ^{-d}$, where $\tilde C(d)=C(d)+\beta^2 C_0(d)$. 

To conclude, this implies that  $\exists k_0, C_1(d)>0$ and $C_2(d)>0$, such that for all $k \geq k_0$ it holds that 
$$C_1 k^{-d} \leq \lambda_k \leq C_2 k^{-d} $$
and also, unless $\beta=0$, for all $k \geq 0$ $$\lambda_k > 0.$$
\end{proof}

Next, we prove the second part of Theorem~\ref{app:thm:NTKdecay} that relates to deep FC networks with bias, $\krfcb$, i.e. we prove the following lemma.

\begin{lemma}
Let $\x,\z \in \mathbb{S}^{d-1}$ and $\krfcb(\x^T\z)$ as defined in Appendix \ref{app:NTK_formulas}. Then
\begin{enumerate}
\item $\krfcb$ decomposes according to \eqref{app:eq:sh} with $\lambda_k>0$  for all $k \ge 0$
\item $\exists k_0$ such that $\forall k>k_0$ it holds that
 $C_3 k^{-d} \leq \lambda_k$ in which
  $C_3>0$ depends on the dimension $d$
\item $\HNTKFCbl{L-1} \subseteq \HNTKFCb$
\end{enumerate}
\end{lemma}

\begin{proof}
Following Lemma \ref{lemma:app:ntk_on_sphere_zonal}, it holds that $\krfcb$ is zonal, and therefore can be decomposed according to \eqref{app:eq:sh}. In order to prove the lemma we look at the recursive formulation of the NTK kernel, i.e., 
\begin{equation}
\label{eq:app:recursive}\krfcbl{l+1} = \krfcbl{l}\dot{\Sigma}^{(l)} + \Sigma^{(l)} + \beta^2.\end{equation}
Now, following Lemma 17 in \cite{bietti2019inductive}
all of  the eigenvalues of $\dot{\Sigma}^{(l)}$ are positive, including $\lambda_0 > 0$. This implies that the constant function $g(\x)\equiv 1 \in \hh_{\dot{\Sigma}^{(l)}}$. 

Now, we use the norm multiplicity inequality in Sec. \ref{sec:app:preliminaries} and show that $\hh_{\krfcbl{l}} \subseteq 
\hh_{\krfcbl{l} \cdot \dot{\Sigma}^{(l)}}$. Let $f \in \hh_{\krfcbl{l}}$, i.e., $\norm{f}_{\hh_{\krfcbl{l}}} < \infty$. We showed that $1 \in \hh_{\dot{\Sigma}^{(l)}}$. Therefore, $\norm{f \cdot 1}_{\hh_{\krfcbl{l} \cdot \dot{\Sigma}^{(l)}}} \leq \norm{f}_{\hh_{\krfcbl{l}}} \norm{1}_{\hh_{\dot{\Sigma}^{(l)}}}<\infty$, implying that $f \in \hh_{\krfcbl{l} \cdot \dot{\Sigma}^{(l)}}$.

Finally, according to the kernel sum inclusion in Sec. \ref{sec:app:preliminaries}, relying on the recursive formulation \eqref{eq:app:recursive} we have 
$\hh_{\krfcbl{l}} \subseteq \hh_{\krfcbl{l} \cdot \dot{\Sigma}^{(l)}} \subseteq \hh_{\krfcbl{l+1}}$. Therefore, 

\begin{equation}
\label{app:eq:inclusion}
\HNTKFCbl{2} \subseteq \ldots \subseteq \HNTKFCbl{L-1} \subseteq \HNTKFCb.
\end{equation}

This completes the proof, by using Aronszan's inclusion theorem as follows. Since  $H^{k^{FC(2)}}\subseteq H^{k^{FC(L)}}$, then by Aronszajn's inclusion theorem $\exists s > 0$ such that $\krfctwo << s^2 \krfcb$.
Since the kernels are zonal on the sphere (with uniform distribution of the data) their corresponding  RKHS share the same eigenfunctions, namely the spherical harmonics. 

Therefore,  for all $k\geq 0$ it holds 
$$ s^2\lambda^{\krfcb}_k \geq \lambda^{\krfctwo}_k>0$$
and for $k\rightarrow \infty $
it holds that 
\begin{align*}
     s^2\lambda^{\krfcb}_k \geq \lambda^{\krfctwo}_k  \geq \frac{C_1}{k^{d}}
\end{align*}
completing the proof.
 
\end{proof}

\section{Laplace Kernel in $\Spdm$}
\label{sec:Laplace_kernel}

The Laplace kernel $\kr(\x,\y)=e^{-\bar c\norm{\x-\y}}$ restricted to the sphere $\Spdm$ is defined as
\begin{align}\label{exponential_kernel_def}
    K(\x,\y)=\kr(\x^T\y)=e^{-c\sqrt{1-x^Ty}}
\end{align}
where $c>0$ is a tuning parameter. We next prove an asymptotic bound on its eigenvalues.

\begin{theorem} \label{app:thm:exp_decay}
Let $\x,\y \in \Spdm$ and $\kr(\x^T\y)=e^{-c\sqrt{1-\x^T\y}}$ be the Laplace kernel, restricted to $\Spdm$. Then $\kr$ can be decomposed as in \eqref{app:eq:sh} with the eigenvalues $\lambda_k$ satisfying $\lambda_k>0$ for all $k \ge 0$ and $\exists k_0$ such that $\forall k>k_0$ it holds that:
\begin{equation*}
    B_1 k^{-d} \leq \lambda_k \leq B_2 k^{-d}
\end{equation*}
where $B_1,B_2>0$ are constants that depend on the dimension $d$ and the parameter $c$.
\end{theorem}

Our proof relies on several supporting lemmas.
\begin{lemma} \label{app:lemma:stein} (\cite{stein2016introduction} Thm 1.14 page 6)
For all $\alpha > 0$ it holds that  
\begin{equation}
\label{eq:int_fourier}
\int_{\Real^d} e^{-2\pi\norm{\x}\alpha} e^{-2\pi i \tr \cdot \x}d \x = c_d \frac{\alpha}{(\alpha^2 + \norm{\tr}^2)^{(d+1)/2}},\end{equation}
where $c_d = \Gamma(\frac{d+1}{2})/(\pi^{(d+1)/2})$
\end{lemma}

\begin{lemma}\label{sup_exp_3}
Let $f(\x)=e^{-c\norm{\x}}$ with $\x \in \Rd$. Then, its Fourier transform  $\Phi(\w)$ with $\w \in \Rd$ is 
 $\Phi(\w)=\Phi(\|\w\|) =C(1+\norm{\w}^2/c^2)^{-(d+1)/2}$ for some constant $C>0$.
\end{lemma}

\begin{proof}
To calculate the Fourier transform we need to calculate the following integral
$$\Phi(\w) = \frac{1}{(2\pi)^d} \int_{\Real^d} e^{-c\norm{\x}} e^{- i  \x \cdot \w} d\x.$$
According to the Lemma~\ref{app:lemma:stein}, plugging $\alpha = \frac{c}{2 \pi}$ and $\tr = \frac{\w}{2 \pi}$ into \eqref{eq:int_fourier}
yields
$$\Phi(\w)=  c_d \frac{c}{(c^2+\norm{\w}^2)^{(d+1)/2}} = \frac{c_d}{c^{(d+1)}} \frac{1}{\left(1+\frac{\norm{\w}^2}{c^2}\right)^{(d+1)/2}}= C \left(1+\frac{\norm{\w}^2}{c^2}\right)^{-(d+1)/2}$$
with $C=\frac{c_d}{c^{(d+1)}} > 0$.

\end{proof}

\begin{lemma}\label{sup_exp_1}
(\cite{narcowich2002scattered} Thm. 4.1) 
Let $f(\x)$ be defined as $f(\|\x\|)$ for all $\x\in \Rd$, and let $\Phi(\w)=\Phi(\|\w\|)$ denote its Fourier Transform in $\Rd$. Then, its corresponding kernel on $\Spdm$  is defined as the restriction $\kr(\x^T\y)=f(\|\x-\y\|)$ with $\x, \y \in \Spdm$. By Mercer's Theorem the spherical harmonic expansion of $\kr(\x^T\y)$ is of the form
\begin{align*}
    \kr(\x^T\y)=\sum _{k=0}^\infty \lambda _k \sum _{j=1}^{N(d,k)}Y_{k,j}(\x)Y_{k,j}(\y).
\end{align*}
Then, the eigenvalues in the spherical harmonic expansion $\lambda_k$ are related to the Fourier coefficients of $f$, $\Phi(t)$, as follows
\begin{align}
\label{eq:fourier_2_harmonic}
    \lambda_k=\int_o^\infty t\Phi(t)J^2_{k+\frac{d-2}{2}}(t)dt,
\end{align}
where $J_v(t)$ is the usual Bessel function of the first kind of order $v$.
\end{lemma}

Having, these supporting Lemmas, we can now prove {\bf Theorem \ref{app:thm:exp_decay}}.
\begin{proof}
First, $\kr(\cdot,\cdot)$ is a positive zonal kernel and hence can be written as 
\begin{align*}
    \kr(\x^T\y)=\sum _{k=0}^\infty \lambda _k \sum _{j=1}^{N(d,k)}Y_{k,j}(\x)Y_{k,j}(\y).
\end{align*}

Next, to derive the bounds we plug the Fourier coefficients, $\Phi(\omega)$, computed in Lemma \ref{sup_exp_3}, into the expression for the harmonic coefficients, $\lambda_k$ \eqref{eq:fourier_2_harmonic}, obtaining
\begin{align*}
    \lambda_k = C \int_0^\infty \frac{t}{\left(1+\frac{t^2}{c^2}\right)^{\frac{d+1}{2}}} J^2_{k+\frac{d-2}{2}}(t) dt.
\end{align*}
Applying a change of variables $t=cx$ we get
\begin{align} \label{change_of_var}
    \lambda_k = c^2 C \int_0^\infty \frac{x}{(1+x^2)^{\frac{d+1}{2}}} J^2_{k+\frac{d-2}{2}}(cx)dx.
\end{align}
We next bound this integral from both above and below. To get an upper bound we observe that for $x\in [0,\infty )$ $x^2<1+x^2$, implying that $x(1+x^2)^{-(d+1)/2}< x^{-d}$, and consequently
\begin{align*}
    \lambda_k< c^2C \int_0^\infty x^{-d}J^2_{k+\frac{d-2}{2}}(cx)dx:=c^2CA(k,d,c).
\end{align*}
The above integral $A(k,d,c)$ was computed in \cite{watson1995treatise} (Sec. 13.41 page 402 with $a:=c$, $\lambda := d$, and $\mu = \nu := k+(d-2)/2$) which gives
\begin{align}\label{watson_calc}
    A(k,d,c) = \int_0^\infty  x^{-d} J^2_{k+\frac{d-2}{2}}(cx) dx = \frac{(\frac{c}{2})^{d-1} \Gamma(d) \Gamma (k-\frac{1}{2})}{2\Gamma^2(\frac{d+1}{2}) \Gamma(k+d-\frac{1}{2})}.
\end{align}
Using Stirling's formula $\Gamma (x) = \sqrt{2\pi}x^{x-1/2} e^{-x} (1+O(x^{-1}))$ as $x\rightarrow \infty$. Consequently, for sufficiently large $k >> d$
\begin{align}
\label{A(k,d,c)_asym}
    \lambda_k &< c^2C A(k,d,c) = c^2C \frac{(\frac{c}{2})^{d-1} \Gamma(d) \Gamma (k-\frac{1}{2})}{2 \Gamma^2(\frac{d+1}{2})
    \Gamma(k+d-\frac{1}{2})} \nonumber \\ 
    & \sim c^2C \frac{(\frac{c}{2})^{d-1} \Gamma(d)}{2\Gamma^2( \frac{d+1}{2})} \cdot \frac{(k-\frac{1}{2})^{k-1} e^{-k+\frac{1}{2}}}{(k + d - \frac{1}{2})^{k+d-1} e^{-k-d+\frac{1}{2}}} (1+O(k^{-1})) \nonumber \\ 
    &= B_2 k^{-d},
\end{align}
where $B_2$ depends on $c$, $C$ and the dimension $d$.

We use again the relation \eqref{change_of_var} to derive a lower bound for $\lambda_k$. First, note that since $t,1+t^2,J^2_v(t)$ are all non-negative for $t \in [0,\infty)$ and therefore 
\begin{align*}
    \lambda_k & \geq  c^2C \int_1^\infty  \frac{x}{(1+x^2)^{\frac{d+1}{2}}} J^2_{k+\frac{d-2}{2}}(cx) dx \geq c^2C  \int_1^\infty \frac{1}{2^{\frac{d+1}{2}} x^{d}} J^2_{k+\frac{d-2}{2}}(cx)dx\\ 
    &= \frac{Cc^2}{2^{\frac{d+1}{2}}}  \left(\int_0^\infty x^{-d}J^2_{k+\frac{d-2}{2}}(cx)dx-\int_0^1 x^{-d}J^2_{k+\frac{d-2}{2}}(cx)dx\right)\\
    &= \frac{Cc^2}{2^{\frac{d+1}{2}}}  \int_0^\infty  x^{-d}J^2_{k+\frac{d-2}{2}}(cx) dx \left(1-\frac{\int_0^1 x^{-d}J^2_{k+\frac{d-2}{2}}(cx) dx}{\int_0^\infty x^{-d}J^2_{k+\frac{d-2}{2}}(cx) dx}\right)\\
    &= \frac{Cc^2}{2^{\frac{d+1}{2}}} A(k,d,c)\left(1-\frac{B(k,d,c)}{A(k,d,c)} \right),
\end{align*}
where $B(k,d,c) :=  \int_0^1 x^{-d} J^2_{k+\frac{d-2}{2}}(cx) dx$. The first integral, $A(k,d,c)$, was shown in \eqref{A(k,d,c)_asym} to converge asymptotically to $B_2 k^{-d}$.
To bound the second integral, $B(k,d,c)$, we use an inequality from \cite{watson1995treatise} (Section 3.31, page 49), which states that for $v,t\in \mathbb{R}$, $v>-\frac{1}{2}$, 
\begin{align*}
    |J_v(t)|\leq \frac{2^{-v}t^v}{\Gamma(v+1)}.
\end{align*}
This gives an upper bound for $B(k,d,c)$ 
\begin{align*}
    B(k,d,c) = \int_0^1 x^{-d}J^2_{k+\frac{d-2}{2}}(cx)dx\leq \int_0^1 x^{-d}\frac{2^{-2(k+\frac{d-2}{2})}(cx)^{2(k+\frac{d-2}{2})}}{\Gamma^2(k+\frac{d}{2})}dx \leq \frac{(\frac{c}{2})^{2(k+\frac{d-2}{2})}}{\Gamma^2(k+\frac{d}{2})}.
\end{align*}
Applying Stirling's formula we obtain $B(k,d,c)\leq O\left(\frac{(\frac{ce}{2})^{2(k+\frac{d}{2})}(k+d)}{(k+\frac{d}{2})^{2(k+\frac{d}{2})}}\right)$, which implies that as $k$ grows, $\frac{B(k,d,c)}{A(k,d,c)}\rightarrow 0$. Therefore, asymptotically for large $k$
\begin{align*}
    \lambda_k \geq \frac{Cc^2}{2^{\frac{d+1}{2}}} A(k,d,c)\left(1-\frac{B(k,d,c)}{A(k,d,c)}\right) \geq  \frac{Cc^2}{2^{\frac{d+1}{2}}} A(k,d,c),
\end{align*}
from which we conclude that $\lambda_k > B_1k^{-d}$, where the constant $B_1$ depends on $c$, $C$, and $d$.
We have therefore shown that there exists $k_0$ such that $\forall k>k_0$ \begin{align*}
    B_1 k^{-d} \leq \lambda_k\leq B_2 k^{-d}.
\end{align*}
Finally, to show that $\lambda _k>0$ for all $k \geq 0$ we use again \eqref{eq:fourier_2_harmonic} in Lemma \ref{sup_exp_1} which states that
\begin{align*}
    \lambda_k=\int_0^\infty t\Phi(t)J^2_{k+\frac{d-2}{2}}(t)dt.
\end{align*}
Note that in the interval $(0,\infty)$ it holds that $t>0$ and $\Phi(t)>0$ due to Lemma \ref{sup_exp_3}.  Therefore $\lambda_k=0$ implies that $J^2_{k+\frac{d-2}{2}}(t)$ is identically 0  on $(0,\infty)$, contradicting the properties of the Bessel function of the first kind. Hence, $\lambda_k>0$ for all $k$.
\end{proof}

\subsection{Proof of main theorem}
\begin{theorem}  
Let $\HExp$ denote the RKHS for the Laplace kernel restricted to $\Spdm$, and let $\HNTKFCb$ denote the NTK corresponding to a FC network with $L$ layers with bias, restricted to $\Spdm$, then $\HExp = \HNTKFC \subseteq \HNTKFCb$.
\end{theorem} 
\begin{proof}
Let $\lExp_k$, $\lNTKFCbTwo_k$, and $\lNTKFCb_k$ denote the eigenvalues of the three kernel, $\krexp$, $\krfctwo$, and $\krfcb$ in their Mercer's decomposition, i.e., 
\begin{align*}
    \kr(\x^T\z)=\sum _{k=0}^\infty \lambda _k \sum _{j=1}^{N(d,k)}Y_{k,j}(\x)Y_{k,j}(\z).
\end{align*}

Denote by $k_0$ the smallest $k$ for which Theorems \ref{app:thm:NTKdecay} and \ref{app:thm:exp_decay} hold simultaneously. We first show that $\HExp \subseteq \HNTKFC$. Let $f(\x) \in \HExp$, and let $f(\x)=\sum_{k=0}^\infty\sum _{j=0}^{N(d,k)} \alpha _{k,j} Y_{k,j}(\x)$
denote its spherical harmonic decomposition. Then $\|f\|_{\HExp} < \infty$ implies, due to Theorem~\ref{app:thm:exp_decay}, that
\begin{align*}
     \sum_{k=k_0}^\infty\sum_{j=0}^{N(d,k)} \frac{1}{B_2}k^d\alpha_{k,j}^2 \leq \sum_{k=k_0}^\infty\sum_{j=0}^{N(d,k)} \frac{\alpha_{k,j}^2}{\lExp_k} < \infty.
\end{align*}
Combining this with Theorem~\ref{app:thm:NTKdecay}, and recalling that $\lNTKFCbTwo_k > 0$ for all $k \geq 0$), we have
\begin{align*}
    \sum_{k=k_0}^\infty\sum_{j=0}^{N(d,k)}\frac{\alpha_{k,j}^2}{\lNTKFCbTwo_k} \leq  \sum_{k=k_0}^\infty\sum_{j=0}^{N(d,k)}\frac{1}{C_1}k^d\alpha_{k,j}^2 = \frac{B_2}{C_1}\sum_{k=k_0}^\infty\sum_{j=0}^{N(d,k)}\frac{1}{B_2}k^d\alpha_{k,j}^2 < \infty,
\end{align*}
implying that $\norm{f}^2_{\HNTKFC} < \infty$, and so $\HExp \subseteq \HNTKFC$. Similar arguments can be used to show that $\HNTKFC \subseteq \HExp$, proving that $\HNTKFC = \HExp$. Finally, following the inclusion relation \eqref{app:eq:inclusion} the theorem is proved. 
\end{proof}

\section{NTK in $\Rd$}
\label{Appendix:NTK_Rd}

In this section we denote $r_x=\norm{\x}$, $r_z=\norm{\z}$ and by $\hat \x=\x/r_x$, $\hat \z=\z/r_z$.
We first prove Theorem \ref{app:thm:k_homogeneous} and as a consequence Lemma \ref{app:lemma:NTK_sphere} is proved.

\begin{theorem}
\label{app:thm:k_homogeneous}
Let $\krfc(\x,\z), \krfcb(\x,\z)$, $\x, \z \in \Real^d$, denote the NTK kernel with $L$ layers without bias and with bias initialized at zero, respectively. It holds that 
(1) Bias-free $\krfc$ is homogeneous of order 1. (2)  Let $\krb = \krfcb - \krfc$. Then, $\krb$ is homogeneous of order 0.
\end{theorem}

\begin{lemma}
   \label{app:lemma:NTK_sphere}
    Let $\krfcb(\x,\z)$, $\x,\z \in \Spdm$, denote the NTK kernels for FC  networks with $L \ge 2$ layers, possibly with bias initialized with zero. This kernel is zonal, i.e., $\krfcb(\x,\z)=\krfcb(\x^T\z)$. 
\end{lemma}

To that end, we first prove the following supporting Lemma.
\begin{lemma}
For $\x, \z \in \Real^d$ it holds that $$ \Theta^{(L)}(\x,\z) = r_x r_z \Theta^{(L)} \left(\hat\x,\hat\z \right) = r_x r_z \Theta^{(L)}(\hat\x^T\hat\z),$$ where $\Theta^{(L)}=\krfcl{L+1}$, as defined in Appendix \ref{app:NTK_formulas}.
\end{lemma}
\begin{proof}
We prove this by induction over the recursive definition of $\krfcl{L+1} =\Theta^{(L)}(\x,\z)$. Let $\x, \z \in \Rd$, then by definition $$\Theta^{(0)}(\x,\z) = \x^T\z = r_x r_z \Theta^{(0)}\left(\hat\x,\hat\z\right) = r_x r_z \Theta^{(0)}\left(\hat\x^T \hat\z \right)$$
and 
$$\Sigma^{(0)}(\x,\z) = \x^T\z = r_x r_z \Sigma^{(0)}\left(\hat\x,\hat\z\right) = r_x r_z \Sigma^{(0)}\left(\hat\x^T\z\right)$$ 
Assuming the induction hypothesis holds for $l$, i.e., 
$$ \Theta^{(l)}(\x,\z) = r_x r_z \Theta^{(l)}\left(\hat\x,\hat\z\right) = r_x r_z \Theta^{(l)}\left(\hat\x^T \z\right) $$
and 
$$ \Sigma^{(l)}(\x,\z) = r_x r_z \Sigma^{(l)}\left(\hat\x,\hat\z\right) = r_x r_z \Sigma^{(l)}\left(\hat\x^T\hat\z\right) $$
we prove that those equalities  are also true for $l+1$.

By the definition of $\lambda^{(l)}$ \eqref{eq:lambda} and the induction hypothesis for $\Sigma^{(l)}$ we have that
\begin{align*}
\lambda^{(l)}(\x,\z) &= \frac{\Sigma^{(l)}(\x,\z)}{\sqrt{\Sigma^{(l)}(\x,\x)\Sigma^{(l)}(\z,\z)}} = 
\frac{\Sigma^{(l)}\left(\hat\x,\hat\z\right)}{\sqrt{\Sigma^{(l)} \left(\hat\x_i,\hat\x \right) \Sigma^{(l)}\left(\hat\z,\hat\z \right)}} = \lambda^{(l)}\left(\hat\x,\hat\z\right) = \lambda^{(l)}\left( \hat\x^T \hat\z \right) 
\end{align*}
Plugging this result in the definitions of $\Sigma$ \eqref{eq:sigma} and $\dot\Sigma$ \eqref{eq:sigmap},  using the induction hypothesis we obtain
\begin{eqnarray}
\Sigma^{(l+1)}(\x,\z) &=& r_x r_z \Sigma^{(l+1)}\left(\hat\x,\hat\z \right) = r_x r_z \Sigma^{(l+1)}\left(\hat\x^T \hat\z \right) \nonumber \\
\dot\Sigma^{(l+1)}(\x,\z) &=& \dot\Sigma^{(l+1)}\left(\hat\x,\hat\z\right) = \dot\Sigma^{(l+1)}\left(\hat\x^T \hat\z\right)
\label{eq:sigdot0}
\end{eqnarray}
Finally, using the recursion formula \eqref{eq:ntkdeep} ($\beta =0$) and the induction hypothesis for $\Theta^{(l)}$, we obtain 
$$ \Theta^{(l+1)}(\x,\z) = r_x r_z \Theta^{(l+1)}\left(\hat\x,\hat\z\right) = r_x r_z \Theta^{(l+1)}\left(\hat\x^T \hat\z \right)$$ 
\end{proof}

A corollary of this Lemma is that $\krfc$ is homogeneous of order 1  in $\Real^d$, proving the first part of Theorem \ref{app:thm:k_homogeneous}. Also,  it is homogeneous of order 0  in $\Spdm$, proving Lemma \ref{app:lemma:NTK_sphere} for $\beta=0$.

We next turn to proving the second part of Theorem \ref{app:thm:k_homogeneous}, i.e., that $\krb =  \krfcb - \krfc$ is homogeneous of order 0 in  $\Real^d$. By rewriting the recursive definition of $\krfcb$, shown in Appendix~\ref{app:NTK_formulas}, we can express   $\krb$ in the following recursive manner  $\krbl{1}=\beta^2$, and $\krbl{l+1}=\krbl{l}\dot\Sigma+\beta^2$. Therefore,  $\krb$ is homogeneous of order zero, since it depends only on $\dot\Sigma$, which is by itself homogeneous of order zero \eqref{eq:sigdot0}. This concludes Theorem \ref{app:thm:k_homogeneous}.

Finally, Lemma \ref{app:lemma:NTK_sphere} is proved, since $\krfcb = \krfc + \krb$, and when restricted to $\Spdm$ both components are homogeneous of order 0. 
\medskip

\begin{theorem}
\label{app:thm:eig_outofsphere}
Let $p(r)$ be a decaying density on $[0,\infty)$ such that 
$0 < \int_0^\infty p(r)r^2 dr < \infty$ and $\x,\z\in\Rd$.
\begin{enumerate}
    \item Let $\kr_0(\x,\z)$ be homogeneous of order 1 such that $\kr_0(\x,\z) = r_x r_z \hat\kr_0(\hat\x^T\hat\z)$. Then its eigenfunctions with respect to $p(r_x)$ are given by $\Psi_{k,j} = a r_x Y_{k,j}\left(\hat\x\right)$, where $Y_{k,j}$ are the spherical harmonics in $\Spdm$ and $a\in\Real$.
    \item Let $\kr(\x,\z) = \kr_0(\x,\z)$ + $\kr_1(\x,\z)$ so that $\kr_0$ as in 1 and $\kr_1$ is homogeneous of order 0. Then the eigenfunctions of $\kr$ are of the form $\Psi_{k,j} = \left( a r_x + b \right) Y_{k,j}\left(\hat\x\right)$.
\end{enumerate}
\end{theorem}

\begin{proof}
\begin{enumerate}
    \item Since $\hat\kr_0$ is zonal, its Mercer's representation reads
    \[ \hat\kr_0(\hat\x,\hat\z) = \sum_{k=0}^\infty \lambda_k \sum_{j=1}^{N(d,k)} Y_{k,j}(\hat\x) Y_{k,j}(\hat\z), \]
    where the spherical harmonics $Y_{k,j}$ are the eigenfunctions of $\hat\kr_0$. Consequently, as noted also in \cite{bietti2019inductive},
    \[ \kr_0(\x,\z) = a^2 \sum_{k=0}^\infty \lambda_k \sum_{j=1}^{N(d,k)} r_x Y_{k,j}(\hat\x) r_z Y_{k,j}(\hat\z). \]
    The orthogonality of the eigenfunctions $\Psi_{k,j}(\x)=a r_x Y_{k,j}(\hat\x)$ is verified as follows. Let $\bar p(\x)$ denote a probability density on $\Rd$ such that $\bar p(\x) = p(r_x)/A(r_x)$, where $A(r_x)$ denotes the surface area of a sphere of radius $r_x$ in $\Rd$. Then,
    \[ \int_{\Rd} \Psi_{k,j}(\x) \Psi_{k',j'}(\x) \bar p(\x) d\x = a^2 \int_0^\infty \frac{r_x^{d+1} p(r_x)}{A(r_x)} dr_x \int_{\Spdm} Y_{k,j}(\hat\x) Y_{k',j'}(\hat\x) d\hat\x = \delta_{k,k'} \delta_{j,j'}, \]
    where the rightmost equality is due to the orthogonality of the spherical harmonics and by setting
    \[ a^2 = \left( \int_0^\infty \frac{r_x^{d+1} p(r_x)}{A(r_x)} dr_x \right)^{-1}. \]
    Clearly this integral is positive, and the conditions of the theorem guarantee that it is finite.
    
    
    \item By the conditions of the theorem we can write
    \[ \kr(\x,\z) = r_xr_z\hat\kr_0(\hat\x^T\hat\z) + \hat\kr_1(\hat\x^T\hat\z), \]
    where $\hat\x,\hat\z \in \Spdm$.  On the hypersphere the spherical harmonics are the eigenfunctions of $\kr_0$ and $\kr_1$. Denote their eigenvalues respectively by $\lambda_k$ and $\mu_k$, so that
    \begin{eqnarray}
    \label{eq:eig1}
    \int_{\Spdm} \kr_0(\hat \x^T \hat \z) \bar Y_k(\hat\z) d\hat\z = \lambda_k \bar Y_k(\hat\x) \\
    \label{eq:eig2}
    \int_{\Spdm} \kr_1(\hat \x^T \hat \z) \bar Y_k(\hat\z) d\hat\z = \mu_k \bar Y_k(\hat\x),
    \end{eqnarray}
    where $\bar Y_k(\hat\x)$ denote the zonal spherical harmonics. We next show that the space spanned by the functions $r_x \bar Y_k(\x)$ and $\bar Y_k(\x)$ is fixed under the following integral transform
    \begin{equation}  \label{eq:fixedplane}
        \int_{\Rd} \kr(\x,\z) (\alpha r_z + \beta) \bar Y_k(\hat\z) \bar p(\z) d\z = (ar_x+b) \bar Y_k(\hat\x),
    \end{equation}
    $\alpha,\beta,a,b \in \Real$ are constants. The left hand side can be written as the application of an integral operator $T(\x,\z)$ to a function $\Phi^k_{\alpha,\beta}(\z)=(\alpha r_z + \beta) \bar Y_k(\hat\z)$. Expressing this operator application in spherical coordinates yields
    \[
    T(\x,\z) \Phi^k_{\alpha,\beta}(\z) = \int_0^\infty \frac{p(r_z) r_z^{d-1}}{A(r_z)} dr_z \int_{\hat\z \in \Spdm} (r_x r_z \kr_0(\hat\x^T \hat\z) + \kr_1(\hat\x^T \hat\z)) \, (\alpha r_z + \beta) \bar Y_k(\hat\z) d\hat\z.
    \]
    We use \eqref{eq:eig1} and \eqref{eq:eig2} to substitute for the inner integral, obtaining
    \[
    T(\x,\z) \Phi^k_{\alpha,\beta}(\z) = \int_0^\infty \frac{p(r_z) r_z^{d-1}}{A(r_z)} (\lambda_k r_x r_z + \mu_k) (\alpha r_z+\beta) \bar Y_k( \hat\x) dr_z.
    \]
    Together with \eqref{eq:fixedplane}, this can be written as
    \[
    T(\x,\z) \Phi_{\alpha,\beta}(\z) = \Phi_{a,b}(\x),
    \]
    where
    \begin{eqnarray*}
    \begin{pmatrix} 
    a \\ b
    \end{pmatrix}
    &=&
    \begin{pmatrix}
    \lambda_k & 0 \\
    0 & \mu_k
    \end{pmatrix}
    \begin{pmatrix}
    M_2 & M_1 \\ M_1 & M_0
    \end{pmatrix}
    \begin{pmatrix}
    \alpha \\ \beta
    \end{pmatrix}
\end{eqnarray*}
where $M_q=\int_0^\infty \frac{r_z^{q+d-1} p(r_z)}{A(r_z)} dr_z$, $0 \le q \le 2$. By the conditions of the theorem these moments are finite. This proves that the space spanned by $\{r_x\bar Y(\hat\x), \bar Y(\hat\x)\}$ is fixed under $T(\x,\z)$, and therefore the eigenfunctions of $\krfcb(\x,\z)$ take the form $(\bar a r_x + \bar b) \bar Y(\hat\x)$ for some constants $\bar a,\bar b$.
\end{enumerate}
\end{proof}

The implication of Theorem \ref{app:thm:eig_outofsphere} is that the eigenvectors of $\krfc$ are the spherical harmonic functions, scaled by the norm of their arguments. With bias, $\krfcb$ has up to $2N(d,k)$ eigenfunctions for every frequency $k$, of the general form $(ar_x+b)Y_{k,j}(\hat\x)$ where $a,b$ are constants that differ from one eigenfunction to the next.




\section{Experimental Details}

\subsection{The UCI dataSet}

In this section, we provide experimental details for the UCI dataset. We use precisely the same pre-processed datasets, and follow the same performance comparison protocol as in ~\cite{Arora2020Harnessing}. 

\paragraph{NTK Specifications}
We reproduced the results of ~\cite{Arora2020Harnessing} using the publicly available code\footnote{\url{https://github.com/LeoYu/neural-tangent-kernel-UCI}}, and followed the same protocol as in ~\cite{Arora2020Harnessing}. 
The total number of kernels evaluated in~\cite{Arora2020Harnessing} are $15$ and the SVM cost value parameter~$\mathbf{C}$ is tuned from $10^{-2}$ to $10^4$ by powers of $10$. Hence, the total number of hyper-parameter combinations searched using cross-validation is $105~(15\times7)$.

\paragraph{Exponential Kernels Specifications}
For the Laplace and Gaussian kernels, we searched for $10$ kernel width values~($1/c$) from $2^{-2}\times\nu$ to $\nu$ in the log space with base $2$, where $\nu$ is chosen heuristically as the median of pairwise $l_2$ distances between data points~(known as the \textit{median} trick~\cite{dai2014scalable}). So, the total number of kernel evaluations is $10$. For $\gamma$-exponential, we searched through 5 equally spaced values of $\gamma$ from $0.5$ to $2$. Since we wanted to keep the number of the kernel evaluations the same as for NTK in~\cite{Arora2020Harnessing}, we searched through only three kernel bandwidth values~($1/c$) which are $1$, $\nu$ and \#features~(default value in the \textbf{sklearn} package\footnote{\url{https://scikit-learn.org/stable/modules/generated/sklearn.metrics.pairwise.rbf_kernel.html}}). So, the total number of kernel evaluations is $15~(5\times3)$.

For a fair comparison with~\cite{Arora2020Harnessing}, we swept the same range of SVM cost value parameter~$\mathbf{C}$ as in~\cite{Arora2020Harnessing}, i.e., from $10^{-2}$ to $10^4$ by powers of $10$. Hence, the total number of hyper-parameter search using cross-validation is $70~(10\times7)$ for Laplace and $105~(15\times7)$ for $\gamma$-exponential which is the same as for NTK in ~\cite{Arora2020Harnessing}.

\subsection{Large scale datasets}

We used the experimental setup mentioned in~\cite{rudi2017falkon} and the publicly available code~\footnote{\url{https://github.com/LCSL/FALKON_paper}}. \cite{rudi2017falkon} solves kernel ridge regression~(KRR~\cite{scholkopf2001learning}) using the FALKON algorithm, which solves the following linear system
\begin{align*}
    (K_{nn} + \lambda n I)~\alpha = \hat{\y},
\end{align*}
where $K$ is an $n\times n$ kernel matrix defined by $(K)_{ij} = K(x_i, x_j)$, $\hat{\y} = (y_1, \dots y_n)^T$, and $\lambda$ is the regularization parameter. Refer to~\cite{rudi2017falkon} for more details.
      
In Table~\ref{table:largeData_crossVal}, we provide the hyper parameters chosen with cross validation.

\begin{table*}[!htbp]
  \centering
  \begin{tiny}
  \begin{tabular}{|l|c|c|c|}
    \toprule
        & MillionSongs~\cite{bertin2011million} & SUSY~\cite{baldi2014searching} & HIGGS~\cite{baldi2014searching} \\
        \midrule

    H-$\gamma$-exp.                      & $\gamma=1.4, \sigma=5, \lambda=1e^{-6}$   & $\gamma=1.8, \sigma=5, \lambda=1e^{-7}$   & $\gamma=1.6, \sigma=8, \lambda=1e^{-8}$   \\
    
    H-Laplace                       & $\sigma=3, \lambda=1e^{-6}$   &  $\sigma=4, \lambda=1e^{-7}$ &  $\sigma=8, \lambda=1e^{-8}$  \\
    
    NTK                     & $L=9, \lambda=1e^{-9}$   & $L=3, \lambda=1e^{-8}$   & $L=3, \lambda=1e^{-6}$  \\
    
    H-Gaussian                &      $\sigma=8, \lambda=1e^{-6}$   &  $\sigma=3, \lambda=1e^{-7}$ &  $\sigma=8, \lambda=1e^{-8}$  \\

    \bottomrule
  \end{tabular} 
     \caption{\small Hyper-parameters chosen with cross validation for the different kernels.} 
     \label{table:largeData_crossVal}
\end{tiny}
\end{table*}

\subsection{C-Exp: Convolutional Exponential Kernels}  \label{app:cexp}

Let $\x=(x_1,...,x_d)^T$ and $\z=(z_1,...,z_d)^T$ denote two vectorized images. Let $P$ denote a window function (we used $3 \times 3$ windows). Our hierarchical exponential kernels are defined by $\bar\Theta(\x,\z)$ as follows: 
\begin{eqnarray*}
    \Theta^{[0]}_{ij}(\x,\z) &=& x_i z_j \\
    s^{[h]}_{ij}(\x,\z) &=& \sum_{m \in P} \Theta^{[h]}(x_{i+m},z_{j+m}) + \beta^2 \\
    \Theta^{[h+1]}_{ij}(\x,\z) &=& K(s^{[h]}_{ij}(\x,\z),s^{[h]}_{ii}(\x,\x),s^{[h]}_{jj}(\z,\z)) \\ \bar\Theta(\x,\z) &=& \sum_i \Theta_{ii}^{[L]}(\x,\z)
\end{eqnarray*}
where $\beta \ge 0$ denotes the bias and the last step is analogous to a fully connected layer in networks, and we set
\begin{equation*}
    K(s_{ij},s_{ii},s_{jj}) = \sqrt{s_{ii}s_{jj}} \, \kr\left( \frac{s_{ij}}{\sqrt{s_{ii}s_{jj}}} \right)
\end{equation*}
where $\kr$ can be any kernel defined on the sphere. In the experiments we applied this scheme to the three exponential kernels, Laplace, Gaussian and $\gamma$-exponential. 

\paragraph{Technical details} 
We used the following four kernels:

\textbf{CNTK} \cite{arora2019exact} $L=6, \beta=3$.

\textbf{C-Exp Laplace}. $L=3,\beta=3$,  $\kr(\x^T\z)=a+be^{-c\sqrt{2-2\x^T\z}} $ with $a=-11.491,b=12.606,c=0.048$.
    
\textbf{C-Exp $\gamma-$exponential}. $L=8,\beta=3$,  $\kr(\x^T\z)=a+be^{-c(2-2\x^T\z)^{\gamma/2}} $ with $a=-0.276,b=1.236,c=0.424,\gamma=1.888$. 

\textbf{C-Exp Gaussian}. $L=12,\beta=3$,  $\kr(\x^T\z)=a+be^{-c(2-2\x^T\z)} $ with $a=-0.22,b=1.166,c=0.435$. 

We set $\beta$ in these experiments with cross validation in $\{1,...,10\}$. For each kernel $\kr$ above, the parameters $a,b,c$ and $\gamma$ were chosen using non-linear least squares optimization with the objective $\sum_{u \in U} (\kr(u) - \krfctwo(u))^2$, where $\krfctwo$ is the NTK for a two-layer network defined in \eqref{eq:app:ntk_two_layers} with bias $\beta =1$, and the set $U$ included (inner products between) pairs of normalized $3 \times 3 \times 3$ patches drawn uniformly from the CIFAR images.
The number of layers $L$ is chosen by cross validation.

For the training phase we used 1-hot vectors from which we subtracted 0.1, as in \cite{novak2018bayesian}. For the classification phase, as in \cite{li2019enhanced}, we normalized the kernel matrices such that all the diagonal elements are ones. To avoid ill conditioned kernel matrices we applied ridge regression with a regularization factor of $\lambda=5\cdot 10^{-5}$. Finally, to reduce overall running times, we parallelized the kernel computations on NVIDIA Tesla V100 GPUs.

\end{document}